\documentclass[letterpaper]{article} % DO NOT CHANGE THIS
\usepackage{aaai25}  % DO NOT CHANGE THIS
\nocopyright
\usepackage{times}  % DO NOT CHANGE THIS
\usepackage{helvet}  % DO NOT CHANGE THIS
\usepackage{courier}  % DO NOT CHANGE THIS
\usepackage[hyphens]{url}  % DO NOT CHANGE THIS
\usepackage{graphicx} % DO NOT CHANGE THIS
\usepackage{natbib}  % DO NOT CHANGE THIS AND DO NOT ADD ANY OPTIONS TO IT
\usepackage{caption} % DO NOT CHANGE THIS AND DO NOT ADD ANY OPTIONS TO IT
\usepackage{algorithm}
\usepackage{algorithmic}
\usepackage[figuresright]{rotating}
\usepackage{latexsym}
\usepackage{amssymb}
\usepackage{amsmath}
\usepackage{amsthm}
\usepackage{booktabs}
\usepackage{enumitem}
\usepackage{graphicx}
\usepackage{color}
\usepackage{algorithm}
\usepackage{algorithmic}
\usepackage[table]{xcolor}
\usepackage{multirow}
\newtheorem{theorem}{Theorem}
\newtheorem{lemma}{Lemma}

\newtheorem{definition}{Definition}

\usepackage{newfloat}
\usepackage{listings}
\DeclareCaptionStyle{ruled}{labelfont=normalfont,labelsep=colon,strut=off} % DO NOT CHANGE THIS
\floatstyle{ruled}
\newfloat{listing}{tb}{lst}{}
\floatname{listing}{Listing}
\title{SUMO: Search-Based Uncertainty Estimation for Model-Based Offline Reinforcement Learning}
\author{
    Zhongjian Qiao\textsuperscript{\rm 1},
    Jiafei Lyu\textsuperscript{\rm 1},
    Kechen Jiao\textsuperscript{\rm 1},
    Qi Liu \textsuperscript{\rm 2},
    Xiu Li \textsuperscript{\rm 1}
}
\affiliations{
    \textsuperscript{\rm 1}Tsinghua Shenzhen International Graduate School, Tsinghua University\\
    \textsuperscript{\rm 2}Harbin Institute of Technology Shenzhen
}

\usepackage{bibentry}
\begin{document}

\maketitle

\begin{abstract}
The performance of offline reinforcement learning (RL) suffers from the limited size and quality of static datasets. Model-based offline RL addresses this issue by generating synthetic samples through a dynamics model to enhance overall performance. To evaluate the reliability of the generated samples, uncertainty estimation methods are often employed. However, model ensemble, the most commonly used uncertainty estimation method, is not always the best choice. In this paper, we propose a \textbf{S}earch-based \textbf{U}ncertainty estimation method for \textbf{M}odel-based \textbf{O}ffline RL (SUMO) as an alternative. SUMO characterizes the uncertainty of synthetic samples by measuring their cross entropy against the in-distribution dataset samples, and uses an efficient search-based method for implementation. In this way, SUMO can achieve trustworthy uncertainty estimation. We integrate SUMO into several model-based offline RL algorithms including MOPO and Adapted MOReL (AMOReL), and provide theoretical analysis for them. Extensive experimental results on D4RL datasets demonstrate that SUMO can provide more accurate uncertainty estimation and boost the performance of base algorithms. These indicate that SUMO could be a better uncertainty estimator for model-based offline RL when used in either reward penalty or trajectory truncation. Our code is available and will be open-source for further research and development.
\end{abstract}


\section{Introduction}
% In recent years, reinforcement learning (RL)  has achieved remarkable milestones in various fields, such as games \citep{mnih2015human,ye2020mastering}, autonomous driving \citep{kiran2021deep,sallab2017deep}, and robotics \citep{kober2013reinforcement,nguyen2019review}.
% However, RL requires continuous interaction with the environment to learn the optimal policy. which leads to higher costs and even potential risks in the real world.

Offline reinforcement learning (RL) \citep{levine2020offline,prudencio2023survey} aims to learn the optimal policy from a static dataset collected in advance, avoiding the risks and costs associated with environmental interaction in typical RL \citep{sutton2018reinforcement}. 
Nevertheless, since the datasets cannot cover the entire state-action space, the offline agent cannot accurately estimate the Q-value for out-of-distribution (OOD) samples, ultimately leading to a degradation in the agent's performance. 
% Model-free offline RL addresses this issue through techniques such as regularizing the discrepancy between the learned policy and the behavior policy \citep{kumar2019stabilizing,wu2019behavior}, and conservative Q-value estimation \citep{kumar2020conservative,lyu2022mildly}. However, these methods limit the agent's generalization ability.

\begin{figure}[t]
    \centering
    \includegraphics[width=0.95\linewidth]{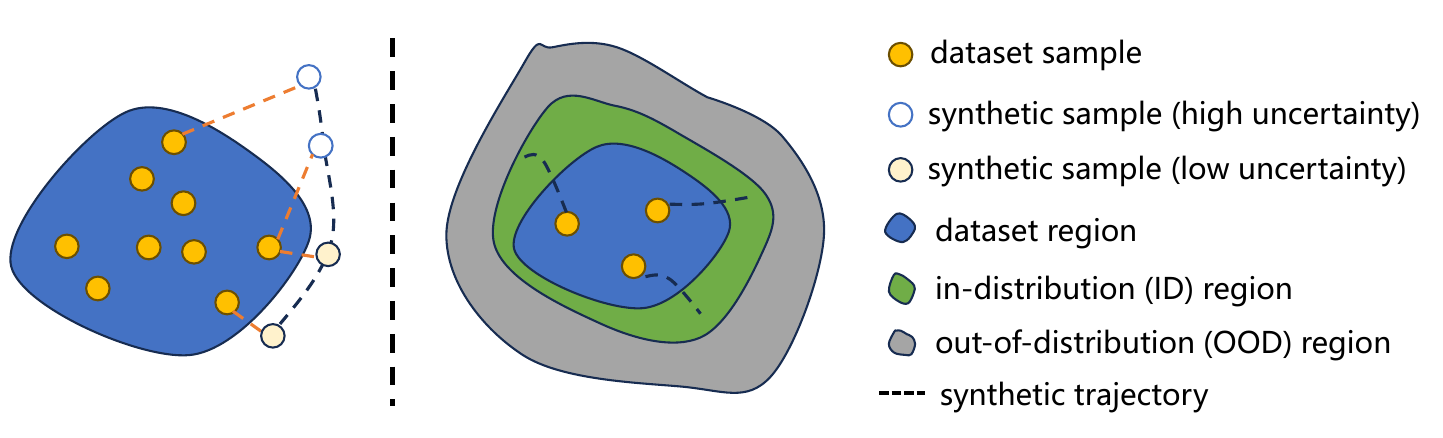}
    \caption{\textbf{Left:} The key idea of SUMO. For a synthetic sample, we calculate its KNN distance within the dataset as the uncertainty estimation. \textbf{Right:} An example of combining SUMO with AMOReL for trajectory truncation. Thanks to the accurate uncertainty estimation provided by SUMO, the agent can explore ID regions while avoiding OOD regions.}
    % \vspace{-0.1cm}
    \label{fig:1}
\end{figure}
% \vspace{-1cm}

Model-based offline RL \citep{yu2020mopo,yu2021combo,kidambi2020morel} employs a promising idea to address OOD issues. By leveraging an environmental dynamics model obtained by supervised learning, the agent can collect samples within the dynamics model and train the policy using these samples. This significantly enhances the performance and generalization of the agent. However, synthetic samples generated by the dynamics model may not be reliable \cite{lyu2022doublecheck}, as the dynamics model's reliability in OOD regions is not guaranteed. Training the policy with unreliable samples can lead to a performance decline. Therefore, evaluating the reliability of generated samples is critical. It is common to assess the reliability with \textit{uncertainty estimation}~\citep{lockwood2022review,an2021uncertainty}. Current model-based offline RL methods often employ model ensemble-based techniques for uncertainty estimation, such as max aleatoric \citep{yu2020mopo} and max pairwise diff \citep{kidambi2020morel}. Subsequently, the estimated uncertainty can be used for trajectory truncation \citep{kidambi2020morel,zhang2023uncertainty} or reward penalties \citep{yu2020mopo} to mitigate OOD issues. However, model ensemble-based uncertainty estimation methods may be unreliable~\citep{yu2021combo} because the models can be poorly learned, making their uncertainty estimation questionable. We wonder: \textit{Can we design a better uncertainty estimation method for model-based offline RL?}
In this paper, we propose a \textbf{S}earch-based \textbf{U}ncertainty estimation method for \textbf{M}odel-based \textbf{O}ffline RL (SUMO). SUMO characterizes the uncertainty of synthetic samples as the cross entropy between model dynamics and true dynamics, which 
% can be shown as
can be shown as 
a more reasonable uncertainty estimation than model ensemble-based estimation. Moreover, the estimated uncertainty for a given sample does not involve extra training of neural networks. In contrast, the uncertainty estimated by model ensemble methods is correlated with the training process and training data distribution since they need to train dynamics models parameterized by neural networks. To estimate the cross entropy practically, we employ a particle-based entropy estimator \citep{singh2003nearest}, transforming the problem into a $k$-nearest neighbor (KNN) search problem, as shown in Figure~\ref{fig:1} (left), that is the reason we call SUMO a search-based method. Furthermore, given the large search space and high data dimensionality, we employ FAISS \cite{johnson2019billion} to ensure efficient KNN search. We note that SUMO is algorithm-agnostic, allowing us to integrate it with any model-based offline RL algorithm which needs uncertainty estimation. For instance, it can be combined with MOPO~\citep{yu2020mopo} or with Adapted MOReL~\citep{kidambi2020morel} (AMOReL, which we discuss in later sections), as shown in Figure~\ref{fig:1} (right).

We integrate SUMO with several off-the-shelf model-based offline RL algorithms like MOPO and AMOReL and theoretically analyze their performance bounds after introducing SUMO. Empirically, we conduct extensive experiments on the D4RL \citep{fu2020d4rl} benchmark, and the experimental results indicate that SUMO can significantly enhance the performance of base algorithms. We also show that SUMO can provide more accurate uncertainty estimation than commonly used model ensemble-based methods. Our contributions can be summarized as follows:
% In our theoretical analysis, we theoretically analyze the performance bounds of AMOReL+SUMO. Empirically, we integrate SUMO with several off-the-shelf model-based offline RL algorithms and conduct extensive experiments on D4RL \citep{fu2020d4rl} benchmark, and the experimental results indicate that SUMO can significantly enhance the performance of base algorithms. We also show that SUMO can provide more accurate uncertainty estimation than commonly used model ensemble-based methods. Our contributions can be summarized as follows:

\begin{itemize}
    \item We propose a novel search-based uncertainty estimation method for model-based offline RL, SUMO.
    \item We combine SUMO with AMOReL and MOPO, and provide theoretical performance bounds.
    \item We empirically demonstrate that SUMO incurs accurate uncertainty estimation and can bring significant performance improvement over the base algorithms on numerous D4RL datasets.
    % \item We empirically show that SUMO incurs more accurate uncertainty estimation than commonly used model ensemble-based methods.
\end{itemize}

\section{Background}
\label{sec:background}
\noindent \textbf{Reinforcement Learning (RL).} We consider a Markov Decision Process (MDP)~\citep{garcia2013markov} modeled by $\mathcal{M}=\langle S,A,r,P,\rho,\gamma \rangle$, where $S$ is the state space, $A$ is the action space, $r$ is the reward function: $S\times A\rightarrow \mathbb{R}$, $P$ is the transition dynamics: $S\times A\times S\rightarrow [0,1]$, $\rho$ is the initial state distribution, and $\gamma\in[0,1)$ is the discount factor. RL aims to get a policy $\pi_\theta(a|s)$ that maximizes the cumulative expected discounted return: $J_\rho(\pi,\mathcal{M})=\max_{\pi_\theta}\mathbb{E}_{\pi_\theta}\left[\sum_{t=0}^\infty \gamma^t r(s_t, a_t) | s_0\sim\rho\right]$. We specify $\mathcal{M}$ to stress that $J_\rho(\pi,\mathcal{M})$ is obtained in the MDP $\mathcal{M}$.

% \begin{equation}
% J(\theta)=\max_{\pi_\theta}\mathbb{E}_{\pi_\theta}\left[\sum_{t=0}^\infty \gamma^t r(s_t, a_t) | s_0\sim\rho\right]
% \end{equation}

\noindent \textbf{Offline RL.} In offline RL, the agent aims to learn the optimal batch-constraint policy based on a static dataset $\mathcal{D}=\{(s_i,a_i,r_i,s_{i+1})\}_{i=1}^N$, where the samples are collected by a (unknown) behavior policy $\mu$, without accessing the environment. Since the dataset can not cover the entire state-action space, the performance of offline RL is often limited.

\noindent \textbf{Model-based Offline RL.}  Model-based offline RL leverages the learned dynamics model $P_{\widehat{\mathcal{M}}}(\cdot|s,a)$ to generate synthetic samples. This allows us to train the policy using both dataset transitions and samples generated by the dynamics model. We denote the model MDP as $\widehat{\mathcal{M}}$, the state distribution probability at timestep $t$ following policy $\pi$ and dynamics $P_{\widehat{\mathcal{M}}}$ as $\mathbb{P}_{\widehat{\mathcal{M}},t}^\pi(s)$, and discounted occupancy measure of $\pi$ under dynamics $P_{\widehat{\mathcal{M}}}$ as $\hat{\rho}^\pi(s,a) := \pi(a|s)\sum_{t=0}^\infty \gamma^t \mathbb{P}_{\widehat{\mathcal{M}},t}^\pi(s)$.
% To enhance the performance and generalization of offline RL, model-based offline RL leverages an environmental dynamics model $P_{\widehat{\mathcal{M}}}(\cdot|s,a)$ to generate synthetic samples. This allows us to train the policy using both samples from the dataset and samples generated by the dynamics model. The dynamics model is trained on the dataset using maximum likelihood estimation (MLE). Therefore, the dynamics model may have generalization errors in OOD regions. It is necessary to assess the reliability of generated samples. We denote the model MDP as $\widehat{\mathcal{M}}$.

\section{Related Work}
\noindent \textbf{Model-free Offline RL.} Model-free offline RL aims to train an agent merely on a static dataset without a dynamics model. The major challenge for model-free offline RL is extrapolation error~\citep{fujimoto2019off,kumar2020conservative,lyu2022mildly} of the Q-function due to the inability to explore. Existing methods introduce conservatism to the policy~\citep{fujimoto2021minimalist,kumar2019stabilizing,fujimoto2019off} or Q-function~\citep{kumar2020conservative,lyu2022mildly,an2021uncertainty,Yang2024ExplorationAA} to tackle the issue. However, the involvement of conservatism restricts the generalization ability of the policy.

\noindent \textbf{Model-based Offline RL.} Model-based offline RL leverages a dynamics model to extend the dataset and enhance the generalization ability. Since the dynamics model may not be accurate on all transitions, conservatism is still necessary for learning a good policy. Some works~\citep{yu2020mopo, kidambi2020morel} incorporate conservatism into the generated samples, while other works~\citep{yu2021combo,rigter2022rambo,liu2023domain} introduce conservatism to Q-function.

% \noindent \textbf{Uncertainty Estimation.} Uncertainty estimation is a widely researched topic in machine learning. Bayesian methods such as Monte-Carlo Dropout~\citep{lee2017training} and Stochastic Variational Inference~\citep{amini2018spatial} characterize the uncertainty by posterior inference. Ensemble methods~\citep{lakshminarayanan2017simple,valdenegro1910deep,wen2020batchensemble} obtain uncertainty by capturing the differences between predictions made by multiple models. There are also some efforts to obtain uncertainty through a single model, such as Evidential Neural Networks~\citep{reinhold2020validating} and Prior Networks~\citep{labonte2020we}.
 
\noindent \textbf{Uncertainty Estimation in offline RL.} 
% Uncertainty estimation is a widely researched topic in machine learning. Bayesian methods such as Monte-Carlo Dropout~\citep{lee2017training} and Stochastic Variational Inference~\citep{amini2018spatial} characterize the uncertainty by posterior inference. Ensemble methods~\citep{lakshminarayanan2017simple,valdenegro1910deep,wen2020batchensemble} obtain uncertainty by capturing the differences between predictions made by multiple models. There are also some efforts to obtain uncertainty through a single model, such as Evidential Neural Networks~\citep{reinhold2020validating} and Prior Networks~\citep{labonte2020we}. 
In model-free offline RL, uncertainty measures the distance of samples from the dataset distribution. EDAC~\citep{an2021uncertainty} leverages an ensemble of Q-networks to estimate the 
uncertainty.
% uncertatiny.
DARL~\citep{zhang2023darl} computes the KNN distances between the samples and the dataset as uncertainty estimation. In model-based offline RL, uncertainty measures the discrepancy between the simulated dynamics and the true dynamics. MOPO~\citep{yu2020mopo} and MOReL~\citep{kidambi2020morel} estimate the uncertainty by model ensemble. MOBILE~\citep{sun2023model} quantifies the uncertainty through model-bellman inconsistency. All of these uncertainty estimation methods rely on model ensemble, while our method does not. Some other uncertainty estimation methods \citep{kim2023model,tennenholtz2022uncertainty} suffer from heavy computational burden and complex implementation, while our method is both efficient and easy to implement. 
%%% Use this environment to include acknowledgements (optional).
%%% This will be omitted in doubleblind mode.

\section{Methodology}

\subsection{Search-based method for Uncertainty Estimation}
In this part, we formally present our search-based uncertainty estimation method for model-based offline RL. 
Firstly, we outline how we model the uncertainty estimation problem as an estimation of cross entropy. Subsequently, we employ a particle-based entropy estimator, which approximates the calculation of cross entropy as a search problem. 
% Finally, we utilize FAISS, an efficient GPU-based method for KNN search, to ensure time efficiency.

We characterize the uncertainty of the model dynamics as its cross entropy against the true dynamics: $\mathcal{H}(P_{\widehat{\mathcal{M}}}(\cdot|s,a),P(\cdot|s,a))=-\sum P_{\widehat{\mathcal{M}}}(\cdot|s,a) \log P(\cdot|s,a)$, where $P_{\widehat{\mathcal{M}}}(\cdot|s,a)$ and $P(\cdot|s,a)$ are the model dynamics and the true dynamics, respectively. It is hard to calculate $\mathcal{H}(P_{\widehat{\mathcal{M}}}(\cdot|s,a),P(\cdot|s,a))$ since we have no knowledge about the true dynamics $P(\cdot|s,a)$. We therefore compute the \emph{cross entropy} between the model dynamics and the dataset dynamics: $\mathcal{H}(P_{\widehat{\mathcal{M}}}(\cdot|s,a),P_d(\cdot|s,a))=-\sum P_{\widehat{\mathcal{M}}}(\cdot|s,a) \log P_d(\cdot|s,a)$, where $P_d(\cdot|s,a)$ denotes the dataset transition dynamics of dataset MDP, which is formally defined below.

\begin{definition}[dataset MDP]
    \label{def:1}
    The dataset MDP is defined by the tuple $\hat{\mathcal{M}_d}=(S,A,r,P_d,\rho_d,\gamma)$, where $S$, $A$, $r$ and $\gamma$ are the same as the original MDP. The dataset transition dynamics $P_d$ is defined as:
    \begin{equation}
        P_d(\cdot|s,a) = \begin{cases} P(\cdot|s,a), \qquad \mathrm{if} \, (s,a)\in \mathcal{D}, \\
    0,\qquad \qquad \qquad\mathrm{otherwise,}  \end{cases}
    \end{equation}
and $\rho_d$ is the state distribution of the dataset.
\end{definition}

\noindent \textbf{Remark:} Intuitively, $P_d$ only accounts for transitions that lie in the dataset; we set the probability for transitions not in the dataset to be zero. Note that the summation of $P_d$ for a given state-action pair $(s,a)$ may be not 1, we can scale $P_d$ accordingly to make it a valid probability density function.

Then based on the definition of cross entropy, we have $\mathcal{H}(P_{\widehat{\mathcal{M}}},P_d)= -\frac{1}{n}\sum_{i=1}^n \log P_d(x_i)$, where $\{x_i\}_{i=1}^n\sim P_{\widehat{\mathcal{M}}}$. According to the particle-based entropy estimator \citep{singh2003nearest}, given a dataset $\{y_j\}_{j=1}^m\sim P_d$, we have $P_d(x_i)= \frac{k\Gamma(d/2+1)}{m\pi^{d/2}R_{i,k,m}^d}$, where $d$ is the dimension of $x_i$, $\Gamma$ is the Gamma function, and $R_{i,k,m}^d = \|x_i-x_i^{k,m}\|_2^d$ represents the distance between $x_i$ and the $k$-nearest neighbor of $x_i$ in $\{y_j\}_{j=1}^m$. Then, we can compute the cross entropy as:
\begin{equation}
    \label{eq1}
    \begin{split}
        \mathcal{H}(P_{\widehat{\mathcal{M}}}, P_d) &= -\frac{1}{n}\sum_{i=1}^n \log \frac{k\Gamma(d/2+1)}{m\pi^{d/2}R_{i,k,m}^d} \\
        &\propto \frac{1}{n}\sum_{i=1}^n \log \|x_i-x_i^{k,m} \|_2
        % J(\theta)&=\max_{\pi_\theta}\mathbb{E}_{\pi_\theta}\left[\sum_{t=0}^\infty \gamma^t r(s_t, a_t) | s_0\sim\rho\right]\\
        % &s.t. ~ ~ \mathbb{E}_{\pi_\theta}\left[\sum_{t=0}^\infty \gamma^t c(s_t, a_t) | s_0\sim\rho\right]\leq d,
    \end{split}
\end{equation}

Notice that Equation (\ref{eq1}) is an uncertainty estimation of the model dynamics. To estimate the uncertainty of a sample $x_i$ generated by the dynamics model, we can assume $P_{\widehat{\mathcal{M}}}$ is a delta distribution, where $P_{\widehat{\mathcal{M}}}(x_i)=\infty$, $P_{\widehat{\mathcal{M}}}(x)=0$ for any $x\neq x_i$. Thus, we transform the problem of estimating the uncertainty of synthetic samples into a KNN search problem. To be specific, given an offline dataset $\mathcal{D}=\{(s_i,a_i,r_i,s_{i+1})\}_{i=1}^N$ and a synthetic sample $(\hat{s},\hat{a},\hat{r},\hat{s}^\prime)$, we can estimate the uncertainty of the synthetic sample as:
\begin{equation}
\label{eq:1}
    u(\hat{s},\hat{a},\hat{s}^\prime) = \log\left( \|(\hat{s}\oplus \hat{a} \oplus \hat{s}^\prime) - (\hat{s}\oplus \hat{a} \oplus \hat{s}^\prime)^{k,N}\|_2\right)
\end{equation}
where $\oplus$ is the vector concatenation operator, $(\hat{s}\oplus \hat{a} \oplus \hat{s}^\prime)^{k,N}$ returns the nearest neighbor of the query sample $(\hat{s},\hat{a},\hat{r},\hat{s}^\prime)$ in the offline dataset, $k$ is the number of neighbors for KNN search, and $N$ is the number of samples in the dataset. To ensure the estimated uncertainty is non-negative, we practically compute the uncertainty as:
\begin{equation}
\label{eq:2}
    u(\hat{s},\hat{a},\hat{s}^\prime) = \log\left( \|(\hat{s}\oplus \hat{a} \oplus \hat{s}^\prime) - (\hat{s}\oplus \hat{a} \oplus \hat{s}^\prime)^{k,N}\|_2 +1\right)
\end{equation}
Note that one can also involve the rewards in Equation (\ref{eq:2}) for computing the uncertainty, which we empirically do not find obvious performance difference in Appendix D.3. To ensure time efficiency, we employ FAISS \citep{johnson2019billion}, an efficient GPU-based KNN search method to estimate the uncertainty. Note that using KNN search to calculate distances between given samples and datasets is a widely used approach in model-free offline RL \citep{zhang2023darl,ran2023policy,lyu2024seaboasimple}. However, they often rely on constraining the distance between the learned policy and the behavior policy, and the concatenated vector only includes dimensions of $s$ and $a$. In contrast, for model-based offline RL, what matters is the discrepancy between the model dynamics and the true dynamics. Therefore, we need to consider the influence of $s^\prime$. 

We then discuss the advantage of SUMO compared with model ensemble-based methods. First, SUMO estimates the uncertainty in an unsupervised learning manner, meaning that the estimated uncertainty is stable and not correlated with the training process. Second, model ensemble-based uncertainty estimation is an empirical approximation. A mismatch can be observed in theoretical analysis of prior works, e.g., $D_{TV}(P_{\widehat{\mathcal{M}}}(\cdot|s,a),P(\cdot|s,a))\leq\alpha$ is required in MOReL, where $D_{TV}$ is the total variation distance and $\alpha$ is the threshold. Model ensemble-based uncertainty estimation can only practically approximate the bound of $D_{TV}$ with $u$. In contrast, SUMO provides an unbiased estimate of cross entropy, relationships between $u$ and $D_{TV}$ can be established using Pinsker's inequality \citep{csiszar2011information}, as shown in Appendix A.2.

\subsection{Integrating SUMO into Existing Methods}

As an uncertainty estimation method, 
SUMO is compatible with 
% algorithm-agnostic, meaning that it can be combined with 
any model-based offline RL algorithm that requires an estimation of uncertainty. There are two typical ways of using uncertainty in model-based offline RL, using the uncertainty as the reward penalty, e.g., MOPO \cite{yu2020mopo}, and using the uncertainty to truncate synthetic trajectories from the learned model, e.g., MOReL \cite{kidambi2020morel}. In this part, we introduce the procedure of combining SUMO with MOPO and Adapted MOReL, respectively.
\subsubsection{MOPO with SUMO}
MOPO adopts the uncertainty as the reward penalty, and we can simply replace the uncertainty estimator as the SUMO estimator. The detailed pseudocode of MOPO+SUMO can be found in Appendix B, and the detailed process can summarized as:

\noindent \textbf{Step1: Learning the Environmental Dynamics Model:} We first construct an environmental dynamics model to simulate the true dynamics. Following previous works~\citep{yu2020mopo,kidambi2020morel}, we model the dynamics model as a multi-layer neural network which outputs a Gaussian distribution over the next state and reward: $\hat{P}_\psi(s^\prime,r|s,a)=\mathcal{N}\left(\mu_{\psi}(s,a),\Sigma_{\psi}(s,a)\right)$. We use an ensemble dynamics model which consists of $N$ ensemble members: $\hat{P}_{\psi}=\{\hat{P}_{\psi_i}\}_{i=1}^N$. Note that we use the ensemble dynamics model solely to reduce the error in model predictions, not for estimating the uncertainty. Then we can train each ensemble member via maximum log-likelihood:
\begin{equation}
\label{eq:3}
    \mathcal{L}_{\psi_i}=\mathbb{E}_{(s,a,r,s^\prime)\sim \mathcal{D}}\left[-\log \hat{P}_{\psi_i}(s^\prime,r|s,a)\right]
\end{equation}
% When generating synthetic samples, we randomly select an ensemble member as the environmental dynamics model for policy rollout.

\noindent \textbf{Step2: Reward Penalty Calculation with SUMO:} We utilize the current policy $\pi$ and the dynamics model $\hat{P}_{\psi}$ to generate synthetic samples. For each generated sample $(s,a,s^\prime)$, we first calculate its unnormalized uncertainty with SUMO via Equation (\ref{eq:2}). We assume that the reward function satisfies: $0\leq r(s,a)\leq r_{max},\forall (s,a)$. This is easy to achieve in practice. Then we normalize $u(s,a,s^\prime)$ to the range $[0,r_{max}]$:
% by dividing $u(s,a,s^\prime)$ by the maximum value of all $u$ so far, and then multiplying by a factor of $r_{max}$:
\begin{equation}
\label{eq:5}
    \hat{u}(s,a,s^\prime) = \frac{u(s,a,s^\prime)}{\max_{(s,a,s^\prime)} u(s,a,s^\prime)} \times r_{max}
\end{equation}
We adopt $\hat{u}$ to penalize the rewards, $\hat{r}=r-\lambda\hat{u}$, where $\lambda$ is a hyperparameter that controls the magnitude of the penalty. We can then add the synthetic samples with reward penalty to the synthetic dataset $\mathcal{D}_{model}$.

\noindent \textbf{Step 3: Model-based Policy Optimization:} We then follow MOPO to train SAC~\cite{haarnoja2018soft} agent with samples drawn from the offline dataset $\mathcal{D}$ and synthetic dataset $\mathcal{D}_{model}$. Typically, we set a sampling coefficient $\eta\in[0,1]$, sampling a proportion of $\eta\mathcal{B}$ from the offline dataset $\mathcal{D}$, and a proportion of $(1-\eta)\mathcal{B}$ from the synthetic dataset $\mathcal{D}_{model}$ given the batch size $\mathcal{B}$. 

\subsubsection{Adapted MOReL with SUMO}
\label{sec:1}
MOReL calculates model ensemble discrepancy as an uncertainty estimation to determine whether a state-action pair $(s,a)$ is within a known region and penalizes samples from unknown regions. We can simply replace the uncertainty estimation method with SUMO to determine whether a transition $(s,a,s^\prime)$ is reliable. In our practical implementation, we make slight modifications to MOReL. Specifically, when encountering unreliable samples, we directly truncate the trajectory instead of applying a constant penalty. This ensures the reliability of the samples used for training. And we update the policy via SAC~\cite{haarnoja2018soft} instead of planning. We refer to this modified version as Adapted MOReL (AMOReL), and divide the process of SUMO with AMOReL into three steps:

% Below, we introduce the detailed implementation process of combining SUMO with AMOReL, which can also be divided into three steps:

\noindent \textbf{Step 1: Constructing the Environmental Dynamics Model:}
Following the same procedure of MOPO+SUMO, we train an ensemble dynamics model $\hat{P}_{\psi}=\{\hat{P}_{\psi_i}\}_{i=1}^N$.

\noindent \textbf{Step 2: Trajectory Truncation with SUMO:} We then utilize the current policy $\pi$ and dynamics model $\hat{P}_\psi$ to generate synthetic samples. We use SUMO to measure the uncertainty of the samples and truncate the synthetic trajectory accordingly. In specific, in the process of generating synthetic trajectories, we apply Equation (\ref{eq:2}) to estimate the uncertainty of each generated sample $(\hat{s},\hat{a},\hat{s}^\prime)$ and set a truncating threshold $\epsilon$. If the uncertainty of any sample exceeds this threshold, we consider the sample unreliable and stop generating the trajectory, adding the generated trajectory to the synthetic dataset $\mathcal{D}_{model}$. In this way, we ensure using only reliable samples for training. It is important to decide how to choose the truncating threshold $\epsilon$, we can choose to set the threshold as a constant; however, since data distributions and dimensions may vary across different datasets, setting it as a constant might not be the optimal choice. Therefore, \textbf{we set the threshold to be the maximum uncertainty among all the dataset samples}. In specific, for a dataset $\mathcal{D}$, we compute the KNN distances for each sample in $\mathcal{D}$ to other samples. We take the maximum uncertainty value as the truncating threshold. For flexibility, we can also multiply the threshold $\epsilon$ by a coefficient $\alpha$ for adjustment:
\begin{equation}
\label{eq:4}
\epsilon=\alpha\cdot\max_{(s,a,s^\prime)\in\mathcal{D}}\log\left(\|(s\oplus a \oplus s^\prime) - (s\oplus a \oplus s^\prime)^{k,N}\|_2 +1\right)
\end{equation}
where $\alpha$ is a hyperparameter, a larger $\alpha$ makes the algorithm more optimistic, using more synthetic samples for training.

\noindent \textbf{Step 3: Model-based Policy Optimization:} We then draw samples from $\mathcal{D}\cup\mathcal{D}_{model}$ with a sampling coefficient $\eta$ to train an SAC agent following \textbf{Step 3} of MOPO+SUMO.
% Once we obtain reliable synthetic samples, we can proceed with policy optimization. In model-free offline RL, we sample a batch of samples of size $\mathcal{B}$ from the dataset $\mathcal{D}$ for training per update step. For model-based offline RL, we have access to not only the real dataset $\mathcal{D}$, but also the synthetic dataset $\mathcal{D}_{model}$, so we draw samples from $\mathcal{D} \cup \mathcal{D}_{model}$ for training. Typically, we can set a sampling coefficient $\eta\in[0,1]$, sampling a proportion of  $\eta\mathcal{B}$ from the real dataset $\mathcal{D}$, and a proportion of $(1-\eta)\mathcal{B}$ from the synthetic dataset $\mathcal{D}_{model}$ given the batch size $\mathcal{B}$. After performing several steps of policy updates, we can turn to Step 2, continuing to collect synthetic samples using the learned policy and the dynamics model.

We summarize the full pseudocode of AMOReL+SUMO in Appendix B.

\subsection{Theoretical Analysis}
\label{sec:4}
% We first make the following assumption for the benefit of further derivation:

% \begin{assumption}
%     \label{assumption:1}
%     Assume that if the tuple $(s,a)$ exists in the offline dataset $\mathcal{D} = \{(s_i,a_i,r_i,s_{i+1})\}_{i=1}^N$, then all possible transition tuples $\{(s,a,r_j,s^\prime_j)\}, s^\prime_j \sim P(\cdot|s,a)$, are also included in the dataset.
%     % Given an offline dataset $\mathcal{D} = \{(s_i,a_i,r_i,s_{i+1})\}_{i=1}^N$, we assume that if a state-action pair $(s,a)$ is contained within the dataset, then all the transition tuples $\{(s,a,r_j,s^\prime_j)\}_{j=1}^M\sim P(\cdot|s,a)$ are also contained within the dataset. 
% \end{assumption}

% \noindent \textbf{Remark:} Assumption \ref{assumption:1} imposes a requirement on the coverage of the dataset. It can be easily satisfied for deterministic environments (and \textbf{most real-world tasks are indeed deterministic}). For stochastic environments, if the dataset has sufficiently broad coverage, the assumption is met almost surely. Note that all of the considered environments in this paper are fully deterministic.
In this part, we provide theoretical analysis for SUMO. Due to space limit, missing proofs are deferred to Appendix A.

\textbf{We first provide theoretical analysis when leveraging SUMO for penalizing rewards in MOPO}. We show that the policy learned by the MOPO+SUMO has a performance guarantee as follows:

\begin{theorem}[Informal]
    \label{theorem:moposumo}
    Denote the behavior policy of the dataset as $\mu$, the uncertainty estimator $u(s,a,s^\prime)$ defined in Equation (\ref{eq:2}), then under some mild assumptions, MOPO+SUMO incurs the policy $\pi$ that satisfies:
    \begin{equation*}
        J_\rho(\pi) \ge J_{\rho}(\mu) - 2\lambda \mathbb{E}_{\hat{\rho}^\pi}[u(s,a,s^\prime)],
    \end{equation*}
\end{theorem}
where $\lambda\in\mathbb{R}$ is the penalty coefficient term in MOPO, $\hat{\rho}^\pi$ is the discounted occupancy measure of policy $\pi$. This theorem states that in the original MDP $\mathcal{M}$, the policy induced by MOPO+SUMO can be at least as good as the behavior policy if the uncertainty is small.

\textbf{We then present theoretical guarantees when using SUMO for trajectory truncation in AMOReL}. We first define the $\epsilon$-Model MDP:

\begin{definition}[$\epsilon$-Model MDP]
    $\epsilon$-Model MDP is defined by the tuple $\widehat{\mathcal{M}}_\epsilon=\langle S,A,r,\hat{P}_\epsilon,\hat{\rho}_\epsilon,\gamma\rangle$, where $S$, $A$, $r$ and $\gamma$ are the same as the original MDP, $\hat{\rho}_\epsilon$ is the state distribution of the learned dynamics model. The transition dynamics $\hat{P}_\epsilon$ is defined as:
    \begin{equation}
        \hat{P}_\epsilon(s^\prime|s,a) = \begin{cases} P_{\widehat{\mathcal{M}}}(s^\prime|s,a), \qquad \mathrm{if} \, u(s,a,s^\prime)<\epsilon, \\
    0,\qquad \qquad \qquad\mathrm{otherwise,}  \end{cases}
    \end{equation}
\end{definition}

Intuitively, $\widehat{\mathcal{M}}_\epsilon$ only contains transitions satisfying $u(s,a,s^\prime)<\epsilon$. In this way, we ensure the reliability of synthetic samples for training. Then we present our theoretical results of performance bounds for $\epsilon$-Model MDP.

\begin{theorem}[Performance bounds]
    \label{theorem:1}
    Denote $\rho_d$ as the state distribution of the dataset. For any policy $\pi$, its return in the $\epsilon$-Model MDP $\widehat{\mathcal{M}}_\epsilon$ and the original MDP $\mathcal{M}$ (with state distribution $\rho$) satisfies:
    \begin{equation*}
        \begin{split}
            J_{\hat{\rho}_\epsilon}(\pi,\widehat{\mathcal{M}}_\epsilon) \geq &J_{\rho}(\pi,\mathcal{M})-\frac{r_{max}}{1-\gamma}( 1+\sqrt{\frac{\epsilon}{2}}+2D_{TV}(\rho,\rho_d)  \\
            &+2D_{TV}(\hat{\rho}_\epsilon,\rho_d)),
        \end{split}
    \end{equation*}
    \begin{equation*}
        \begin{split}
            J_{\hat{\rho}_\epsilon}(\pi,\widehat{\mathcal{M}}_\epsilon) \leq &J_{\rho}(\pi,\mathcal{M})+\frac{r_{max}}{1-\gamma}(\sqrt{\frac{\epsilon}{2}}+2D_{TV}(\rho,\rho_d)  \\
            &+2D_{TV}(\hat{\rho}_\epsilon,\rho_d))
        \end{split}
    \end{equation*}
\end{theorem}

\noindent \textbf{Remark:} Theorem \ref{theorem:1} states that the performance difference for any policy $\pi$ in $\epsilon$-Model MDP and the original MDP is related to three factors: (1) the truncation threshold $\epsilon$, which determines the synthetic samples for training. (2) the total variance distance between the true and dataset state distribution $D_{TV}(\rho,\rho_d)$. (3) the total variance distance between $\epsilon$-Model and dataset state distribution $D_{TV}(\hat{\rho}_\epsilon,\rho_d)$.

\begin{table*}
\caption{Normalized average score comparison between vanilla base algorithms and the version with SUMO, on top of MOPO, AMOReL, MOReL and MOBILE on 15 D4RL MuJoCo datasets, and the version of datasets we use is "-v2". We abbreviate "halfcheetah" as "half", "random" as "r", "medium" as "m", "medium-replay" as "m-r", "medium-expert" as "m-e" and "expert" as "e". We run each algorithm for 1M gradient steps with 5 random seeds. We report the final average performance and $\pm$ captures the standard deviation. Bold numbers with a green background represent the best average scores within each group.}
\label{tab:1}
\centering
\begin{tabular}{c|cc|cc|cc|cc}
\toprule
\multicolumn{1}{c}{\multirow{2}{*}{\textbf{Task Name}}}  & \multicolumn{2}{c}{MOPO}  & \multicolumn{2}{c}{AMOReL} & \multicolumn{2}{c}{MOReL} & \multicolumn{2}{c}{MOBILE}\\ 
\cline{2-9}

 & SUMO & Base & SUMO & Base & SUMO & Base & SUMO & Base \\
\midrule

half-r  & \cellcolor{green!25} \textbf{37.2}$\pm$1.9  & 34.9$\pm$1.4 & \cellcolor{green!25} \textbf{44.2}$\pm$2.1 & 31.8$\pm$2.4 & \cellcolor{green!25}\textbf{37.3$\pm$2.1} & 29.8$\pm$1.2 & 34.9$\pm$2.1 & \cellcolor{green!25}\textbf{37.8$\pm$2.9}\\
hopper-r  & \cellcolor{green!25} \textbf{24.5}$\pm$0.9  & 19.4$\pm$0.7 & 29.7$\pm$0.6 & \cellcolor{green!25} \textbf{32.4$\pm1.2$} & \cellcolor{green!25}\textbf{33.2$\pm$0.7} & 30.1$\pm$1.0  & 30.8$\pm$0.9 & \cellcolor{green!25}\textbf{32.6 $\pm$ 1.2} \\
walker2d-r           & 11.4$\pm$1.3  & \cellcolor{green!25} \textbf{13.1$\pm$1.1}  &    20.3$\pm$0.2          & \cellcolor{green!25} \textbf{21.0$\pm$0.3} & 17.8$\pm$0.6 & \cellcolor{green!25}\textbf{19.4$\pm$0.3} & \cellcolor{green!25}\textbf{27.9$\pm$2.0} & 16.3$\pm$ 4.6 \\
half-m-r & \cellcolor{green!25} \textbf{73.1}$\pm$2.1  & 65.0$\pm$3.3 & \cellcolor{green!25} \textbf{76.8}$\pm$2.7   & 49.6$\pm$2.3 & \cellcolor{green!25}\textbf{67.9$\pm$2.5} &51.2$\pm$1.9  & \cellcolor{green!25}\textbf{76.2$\pm$1.3} & 67.9$\pm$2.0  \\
hopper-m-r      & \cellcolor{green!25} \textbf{65.4}$\pm$3.2 & 38.8$\pm$2.4 & \cellcolor{green!25} \textbf{88.7}$\pm$1.3   & 80.5$\pm$1.0 & \cellcolor{green!25}\textbf{83.9$\pm$1.3} &76.3$\pm$1.0 & \cellcolor{green!25}\textbf{109.9$\pm$1.4} & 104.9$\pm$0.9 \\
walker2d-m-r    & 70.3$\pm$0.6  & \cellcolor{green!25} \textbf{74.8$\pm$1.5} & \cellcolor{green!25} \textbf{65.3}$\pm$2.7   & 46.0$\pm$1.9 & \cellcolor{green!25}\textbf{61.3$\pm$3.1} &48.1$\pm$4.2 & 78.2$\pm$1.5 & \cellcolor{green!25}\textbf{83.9$\pm$1.3} \\
half-m        & 68.9$\pm$2.3  & \cellcolor{green!25} \textbf{73.1$\pm$2.7} & \cellcolor{green!25} \textbf{82.1}$\pm$2.8   & 69.2$\pm$1.2 & 57.9$\pm$1.2 &\cellcolor{green!25}\textbf{62.4$\pm$1.3} & \cellcolor{green!25}\textbf{84.3$\pm$2.4} & 75.1$\pm$1.5  \\
hopper-m  & \cellcolor{green!25} \textbf{74.6}$\pm$1.9 & 45.6$\pm$2.5 & \cellcolor{green!25} \textbf{95.0}$\pm$2.1  & 87.2$\pm$3.4 & 82.1$\pm$1.4 &\cellcolor{green!25}\textbf{84.7$\pm$3.1} & \cellcolor{green!25}\textbf{104.8$\pm$2.1} & 102.9$\pm$1.9 \\
walker2d-m           & \cellcolor{green!25} \textbf{57.3}$\pm$1.6  & 42.3$\pm$0.8 & 67.4$\pm$0.9   & \cellcolor{green!25} \textbf{71.2$\pm$1.3}& \cellcolor{green!25}\textbf{77.1$\pm$3.5} &67.6$\pm$2.2 & \cellcolor{green!25}\textbf{94.1$\pm$2.5} & 89.1$\pm$1.0  \\
half-m-e & \cellcolor{green!25} \textbf{84.1}$\pm$1.4  & 76.6$\pm$1.0 & \cellcolor{green!25} \textbf{99.4}$\pm$3.6 & 90.6$\pm$2.1 & \cellcolor{green!25}\textbf{98.6$\pm$3.5} & 92.3$\pm$4.6 & 106.6$\pm$2.4 & \cellcolor{green!25}\textbf{109.2$\pm$3.8} \\
hopper-m-e  & \cellcolor{green!25} \textbf{88.1}$\pm$1.9 & 69.1$\pm$1.2 & 101.5$\pm$0.4 & \cellcolor{green!25} \textbf{106.2$\pm$1.5} & \cellcolor{green!25}\textbf{105.8$\pm$1.4} &102.4$\pm$0.9 & 107.8$\pm$0.7 & \cellcolor{green!25}\textbf{110.1$\pm$1.3} \\
walker2d-m-e    &\cellcolor{green!25}  \textbf{81.9}$\pm$1.6 & 75.4$\pm$1.1 & \cellcolor{green!25} \textbf{109.6}$\pm$0.7  & 92.3$\pm$0.9 & 86.1$\pm$1.8 & \cellcolor{green!25}\textbf{90.4$\pm$1.4} & \cellcolor{green!25}\textbf{122.8$\pm$0.4} & 115.9$\pm$0.8 \\
half-e & 87.1$\pm$1.2 & \cellcolor{green!25} \textbf{88.7$\pm$1.6} & \cellcolor{green!25} \textbf{112.3}$\pm$2.5 & 103.2$\pm$1.9 & \cellcolor{green!25}\textbf{109.9$\pm$2.1} &105.8$\pm$1.6 & 111.5$\pm$1.5 & \cellcolor{green!25}\textbf{113.1$\pm$2.1}\\
hopper-e & \cellcolor{green!25} \textbf{101.2}$\pm$ 1.8 & 83.9$\pm$0.7 & \cellcolor{green!25} \textbf{105.4}$\pm$0.6 & 94.5$\pm$0.3 & \cellcolor{green!25}\textbf{101.8$\pm$1.9} &92.5$\pm$1.0 & \cellcolor{green!25}\textbf{115.9 $\pm$ 2.9} & 112.4$\pm$3.5\\
walker2d-e & \cellcolor{green!25} \textbf{114.4}$\pm$1.1 & 95.3$\pm$3.4 & 106.3$\pm$1.3 & \cellcolor{green!25} \textbf{107.2$\pm$1.0} & 106.2$\pm$1.5 &\cellcolor{green!25}\textbf{108.3$\pm$2.1} & \cellcolor{green!25}\textbf{116.3$\pm$1.5} & 113.7$\pm$1.1\\
\midrule
\textbf{Average score} & \cellcolor{green!25} \textbf{69.3} & 59.7 & \cellcolor{green!25} \textbf{80.3} & 72.2 &
\cellcolor{green!25}\textbf{75.1} &70.7 & \cellcolor{green!25}\textbf{88.2} & 85.6\\
\bottomrule
\end{tabular}
\end{table*}

\section{Experiment}
\label{sec:experiment}
In our experimental part, we empirically evaluate our method, SUMO. We aim to answer the following questions: (1) How much performance gain can SUMO bring to off-the-shelf model-based offline RL algorithms? (2) Can SUMO provide more accurate uncertainty estimation compared to model ensemble-based methods? (3) How different design choices affect the performance of SUMO? (4) How sensitive is SUMO to the introduced hyperparameters?

For evaluation, we use the D4RL MuJoCo datasets which includes three tasks: \textit{halfcheetah}, \textit{hopper} and \textit{walker2d}. Each task provides five types of offline datasets: \textit{random}, \textit{medium}, \textit{medium-replay}, \textit{medium-expert} and \textit{expert}. We conduct experiments on 15 D4RL MuJoCo datasets to comprehensively evaluate the performance of SUMO across various tasks and datasets of different qualities.

\subsection{Experimental Results on the D4RL Benchmark}
\label{sec:2}
SUMO is an uncertainty estimation method designed for model-based offline RL, it can be seamlessly combined with any model-based offline RL algorithm which needs uncertainty estimation. In this part, we combine SUMO with several model-based offline RL algorithms, including MOPO, AMOReL, MOReL and MOBILE~\cite{sun2023model}. We conduct extensive experiments on widely used D4RL MuJoCo datasets and examine whether SUMO can bring performance improvement to these base algorithms. We run all experiments with 5 different random seeds.

We summarize the experimental results in Table \ref{tab:1}, where we observe that SUMO significantly boosts the performance of base algorithms. Notably, MOPO+SUMO outperforms vanilla MOPO in \textbf{11} out of 15 tasks, AMOReL+SUMO and MOReL+SUMO surpasses the original AMOReL and MOReL in \textbf{10} out of 15 tasks, MOBILE+SUMO achieves a higher score than MOBILE in \textbf{9} out of 15 tasks. Regarding the average score on all 15 tasks, SUMO brings
 an overall performance improvement for all four base algorithms, indicating the versatility of SUMO. We include more comparison with model-based methods without uncertainty estimation such as COMBO~\cite{yu2021combo} and RAMBO~\cite{rigter2022rambo} in Appendix D.5.

We also evaluate our methods on D4RL Antmaze tasks, which are challenging for model-based offline RL methods~\citep{wang2021offline}. Due to space limit, we present the results in Appendix D.1. We find that base model-based RL algorithms all struggle to perform well in the Antmaze domains. When combined with SUMO, the performance of base algorithms has been improved. We believe these further demonstrate the strengths of SUMO.

% \subsection{Comparison with other Uncertainty Estimation Methods}
\subsection{Comparison with other Uncertainty Estimators}
\label{sec:3}

\begin{table*}[!h]
\caption{Spearman rank ($\rho$) and Pearson bivariate ($r$) correlations comparison of SUMO against Max Aleatoric, Max Pairwise Diff and LOO KL Divergence. We choose these two metrics following~\citep{lu2021revisiting}. 
The closer $\rho$ and $r$ are to 1, the better the uncertainty estimation. Each method is evaluated by 5 runs and we report the average $\rho$ and $r$. 
Bolded numbers with a green background represent the method with the highest average 
$\rho$, and bolded numbers with a yellow background represent the method with the highest $r$ on each dataset.}
\label{tab:2}
    \centering
\begin{tabular}{l | cc | cc | cc | cc}
\toprule
\multicolumn{1}{c}{\multirow{2}{*}{\textbf{Task Name}}} &
  \multicolumn{2}{c}{\textbf{Max Aleatoric}} &
  \multicolumn{2}{c}{\textbf{Max Pairwise Diff}} & 
  \multicolumn{2}{c}{\textbf{LOO KL Divergence}} &
  \multicolumn{2}{c}{\textbf{SUMO(Ours)}}
  \\ \cline{2-9} 
\multicolumn{1}{c}{} &
  $\mathbf{\rho}$ &
  $r$ &
  $\mathbf{\rho}$ &
  $r$ &
  $\mathbf{\rho}$ &
  $r$ &
  $\mathbf{\rho}$ &
  $r$ \\ \hline
  halfcheetah-random-v2 & 0.63 & 0.58 & 0.58 & 0.47 & 0.10 & 0.06 & \cellcolor{green!25}\textbf{0.84} & \cellcolor{yellow!50}\textbf{0.83} \\
  hopper-random-v2 & 0.76 & 0.71 & \cellcolor{green!25}\textbf{0.85} & \cellcolor{yellow!50}\textbf{0.77} & 0.13 & 0.09 & 0.82 & 0.74 \\
  walker2d-random-v2 & \cellcolor{green!25}\textbf{0.75} & 0.62 & 0.71 & 0.58 & 0.14 & 0.10 & 0.73 & \cellcolor{yellow!50}\textbf{0.70} \\
  halfcheetah-medium-replay-v2 & 0.70 & 0.65 & 0.73 & 0.67 & 0.21 & 0.16 & \cellcolor{green!25}\textbf{0.82} & \cellcolor{yellow!50}\textbf{0.68} \\
  hopper-medium-replay-v2 & 0.73 & 0.70 & 0.79 & 0.67 & 0.18 & 0.08 & \cellcolor{green!25}\textbf{0.87} & \cellcolor{yellow!50}\textbf{0.80} \\
  walker2d-medium-replay-v2 & 0.66 & 0.54 & 0.65 & 0.59 & 0.09 & 0.06 & \cellcolor{green!25}\textbf{0.88} & \cellcolor{yellow!50}\textbf{0.86} \\
  halfcheetah-medium-v2 & 0.69 & 0.59 & 0.65 & 0.57 & 0.15 & 0.11 & \cellcolor{green!25}\textbf{0.84} & \cellcolor{yellow!50}\textbf{0.77} \\
  hopper-medium-v2 & 0.74 & 0.66 & 0.70 & 0.63 & 0.11 & 0.05 & \cellcolor{green!25}\textbf{0.86} & \cellcolor{yellow!50}\textbf{0.79} \\
  walker2d-medium-v2 & 0.68 & 0.56 & 0.74 & \cellcolor{yellow!50}\textbf{0.73} & 0.11 & 0.08 & \cellcolor{green!25}\textbf{0.82} & 0.71 \\
  halfcheetah-medium-expert-v2 & 0.67 & 0.63 & 0.70 & 0.68 & 0.13 & 0.07 & \cellcolor{green!25}\cellcolor{green!25}\textbf{0.88} & \cellcolor{yellow!50}\textbf{0.81} \\
  hopper-medium-expert-v2 & 0.71 & 0.68 & \cellcolor{green!25}\textbf{0.78} & \cellcolor{yellow!50}\textbf{0.72} & 0.22 & 0.14 & 0.75 & 0.68 \\
  walker2d-medium-expert-v2 & 0.62 & 0.59 & 0.65 & 0.60 & 0.08 & 0.06 & \cellcolor{green!25}\textbf{0.84} & \cellcolor{yellow!50}\textbf{0.80} \\
  halfcheetah-expert-v2 & \cellcolor{green!25}\textbf{0.79} & \cellcolor{yellow!50}\textbf{0.75} & 0.69 & 0.67 & 0.12 & 0.10 & 0.76 & 0.68 \\
  hopper-expert-v2 & 0.73 & 0.64 & 0.76 & 0.67 & 0.17 & 0.14 & \cellcolor{green!25}\textbf{0.86} & \cellcolor{yellow!50}\textbf{0.80}\\
  walker2d-expert-v2 & 0.63 & 0.57 & 0.65 & 0.58 & 0.09 & 0.05 & \cellcolor{green!25}\textbf{0.83} & \cellcolor{yellow!50}\textbf{0.76} \\
  \bottomrule
\end{tabular}
\end{table*}

In this section, we aim to demonstrate the superiority of SUMO in a more intuitive manner, i.e., whether SUMO can provide more accurate uncertainty estimation compared to widely used model ensemble-based uncertainty estimation methods. This is crucial and necessary for validating our central claim in this paper.
% This is crucial as it can elucidate whether the performance improvement is attributed to the more accurate uncertainty estimation brought about by SUMO.

% In addition to the overall performance improvement, in this section, we aim to demonstrate the superiority of SUMO in a more intuitive manner. Specifically, we want to examine whether SUMO can provide more accurate uncertainty estimation compared to widely used model ensemble-based uncertainty estimation methods. This is crucial as it can elucidate whether the performance improvement is attributed to the more accurate uncertainty estimation brought about by SUMO.

For model ensemble-based uncertainty estimation methods, the core idea is to measure the uncertainty of generated samples by leveraging the diverse predictions of ensemble members regarding the environmental dynamics. We choose the following model ensemble-based uncertainty estimation methods for comparison, which come from recent literature in both offline and online model-based RL:

% \noindent \textbf{Max Aleatoric:} The method used by MOPO. It measures the sample uncertainty as: $u(s,a)=\max_{i=1,...,N}\left\|\Sigma_{\psi_i}(s,a)\right\|_F$, where $\Sigma_{\psi_i}$ represents the covariance matrix predicted by the $i$-th ensemble member, and $\left\|\cdot\right\|_F$ denotes the Frobenius norm. This uncertainty heuristic is related to the aleatoric uncertainty.

% \noindent \textbf{Max Pairwise Diff:} The method used by MOReL. It measures the sample uncertainty as: $u(s,a)=\max_{i,j}\left\|\mu_{\psi_i}(s,a)-\mu_{\psi_j}(s,a)\right\|_2$, where $\mu_{\psi_i}$ is the mean vector predicted by the $i$-th ensemble member. This uncertainty heuristic corresponds to the maximum pairwise difference between ensemble predictions.

% \noindent \textbf{LOO (Leave-One-Out) KL Divergence:} The method used by online model-based RL algorithm M2AC~\citep{pan2020trust}. It measures the sample uncertainty as: $u(s,a)=D_{KL}(\hat{P}_{\psi_i}(\cdot|s,a)||\hat{P}_{\psi_{-i}}(\cdot|s,a))$ where $\hat{P}_{\psi_{-i}}(\cdot|s,a)$ represents the aggregated Gaussian distribution of all ensemble members except the $i$-th member. 

%%%%%%%
\noindent \textbf{Max Aleatoric:} It measures the sample uncertainty as: $u(s,a)=\max_{i=1,...,N}\left\|\Sigma_{\psi_i}(s,a)\right\|_F$, where $\Sigma_{\psi_i}$ is the covariance matrix predicted by the $i$-th ensemble member, and $\left\|\cdot\right\|_F$ denotes the Frobenius norm. This uncertainty heuristic is used in MOPO and is related to the aleatoric uncertainty.

\noindent \textbf{Max Pairwise Diff:} It measures the sample uncertainty as: $u(s,a)=\max_{i,j}\left\|\mu_{\psi_i}(s,a)-\mu_{\psi_j}(s,a)\right\|_2$, $\mu_{\psi_i}$ is the mean vector predicted by the $i$-th ensemble member. This uncertainty heuristic is used in MOReL and measures the maximum pairwise difference between ensemble predictions.

\noindent \textbf{LOO (Leave-One-Out) KL Divergence:} It measures the uncertainty as: $u(s,a)=D_{KL}(\hat{P}_{\psi_i}(\cdot|s,a)||\hat{P}_{\psi_{-i}}(\cdot|s,a))$ where $\hat{P}_{\psi_{-i}}(\cdot|s,a)$ represents the aggregated Gaussian distribution of all ensemble members except the $i$-th member. It is used in M2AC~\citep{pan2020trust}.
%%%%%%%

We then compare the ability of these four methods in detecting out-of-distribution (OOD) samples. In offline RL, OOD samples refer to state-action pairs that lie out of the sample distribution of the dataset. It is difficult to decide whether a transition is truly OOD. Nevertheless, things are easier in model-based offline RL, mainly due to the fact that it involves learning an environmental dynamics model. The difference between the simulated environmental dynamics and the real environmental dynamics can be a signal for detecting OOD samples. In practice, we start from the tuple $(s,a)$ sampled from the offline dataset and generate its next state $\hat{s}^\prime$ via the learned dynamics model. We deem that the transition $(s,a,\hat{s}^\prime)$ is OOD if $\hat{s}^\prime$ significantly differs from the dynamics of the real environment. A good uncertainty estimation method should be sensitive to unrealistic environmental dynamics. In other words, when there is a large error in the predicted environmental dynamics, it should provide a higher uncertainty estimate, and when the error is small, it should have a lower uncertainty estimate.

% We then design experiments to compare the ability of these four methods in detecting OOD samples. In model-free offline RL, OOD samples refer to state-action pairs that lie out of the sample distribution of the dataset. However, in model-based offline RL, we cannot define it in the same way. Model-based offline RL involves generating samples outside the dataset range to improve algorithm generalization. For model-based offline RL, what truly matters is the difference between the simulated environmental dynamics and the real environmental dynamics. Therefore, OOD samples refer to $(s,a,s^\prime)$ tuples where $s^\prime$ significantly differs from the dynamics of the real environment. A good uncertainty estimation method should be sensitive to unrealistic environmental dynamics. In other words, when there is a large error in the predicted environmental dynamics, it should provide a higher uncertainty estimate, and when the error is small, it should provide a lower uncertainty estimate.

Empirically, we conduct experiments on the 9 datasets from D4RL MuJoCo tasks. On each dataset, we first train an ensemble dynamics model via supervised learning. Then we train a policy inside the model following the process of MOPO, without adding any penalty. To generate synthetic samples, we randomly select 100 states from the dataset as initial states. For each initial state, we use the trained policy to perform rollouts in the dynamics model. To ensure the generation of OOD samples, we set the rollout horizon to be 100 (the dynamics model tends to output bad transitions under larger rollout horizon due to compounding error). In total, we obtain 10,000 synthetic transitions for each dataset. For any synthetic transition $(s,a,\hat{s}^\prime)$, we can replay the state-action pair $(s,a)$ in the real environment to obtain the true next state $s^\prime$. Then we can calculate the L2-norm error between $s^\prime$ and $\hat{s}^\prime$: $\left\|s^\prime-\hat{s}^\prime\right\|_2$, and this can reflect the difference between the true environmental dynamics and the simulated dynamics. We can then use the aforementioned methods to estimate the uncertainty of the synthetic transition. For a good uncertainty estimation method, we expect a strong correlation between the estimated uncertainty and the L2-norm error of the next state. The remaining question is how to measure the correlation. Following previous work~\citep{lu2021revisiting}, we use Spearman rank ($\rho$) and Pearson bivariate ($r$) correlations, where $\rho$ captures the rank correlations and $r$ measures the linear correlations. 

For each dataset, we calculate $\rho$ and $r$ on the synthesized 10,000 transitions and present the results in Table~\ref{tab:2}. It is evident that the advantages of SUMO are significant. SUMO achieves the best $\rho$ or $r$ metrics in \textbf{12} out of 15 datasets compared with other model ensemble-based methods. We then conclude that SUMO incurs a better uncertainty estimation.

% We also evaluate the score variances across all datasets. Since model ensemble-based uncertainty estimation methods are parameterized, and different datasets can affect the training of the dynamics model, they exhibit larger score variances. SUMO, being an unsupervised learning-based uncertainty estimation method, is less affected by variations in dataset quality, resulting in smaller variances. Note that though the LOO KL Divergence method has smaller variance in $\rho$ and $r$, its average score is much lower than SUMO. We then conclude that SUMO incurs a better uncertainty estimation.

\subsection{Ablation Study \& Parameter Study}
In this part, we examine different design choices of SUMO and investigate the sensitivity of SUMO to the introduced hyperparameters. For the design choices of SUMO, we mainly focus on two aspects: (a) the components included in the search vector; and (b) the choice of distance measure.

\begin{table}
\caption{Spearman rank ($\rho$) and Pearson bivariate ($r$) correlations comparison of SUMO with different search vectors. We run each experiment by 5 different random seeds and report the average $\rho$ and $r$.}
\label{tab:3}
    \centering
\begin{tabular}{l | cc | cc | cc}
\toprule
\multicolumn{1}{c}{\multirow{2}{*}{\textbf{Task Name}}} &
  \multicolumn{2}{c}{$s\oplus s^\prime$} &
  \multicolumn{2}{c}{$s \oplus a$} & 
  \multicolumn{2}{c}{$s \oplus a\oplus s^\prime$} 
  \\ \cline{2-7} 
\multicolumn{1}{c}{} &
  $\mathbf{\rho}$ &
  $r$ &
  $\mathbf{\rho}$ &
  $r$ &
  $\mathbf{\rho}$ &
  $r$ \\ \hline
  half-m & 0.51 & 0.44 & 0.62 & 0.58 & \cellcolor{green!25}\textbf{0.84} & \cellcolor{yellow!50}\textbf{0.77} \\
  hopper-m & 0.50 & 0.48 & 0.55 & 0.52 & \cellcolor{green!25}\textbf{0.86} & \cellcolor{yellow!50}\textbf{0.79} \\
  walker2d-m & 0.64 & 0.49 & 0.59 & 0.51 & \cellcolor{green!25}\textbf{0.82} & \cellcolor{yellow!50}\textbf{0.71} \\
  half-m-e & 0.57 & 0.54 & 0.63 & 0.61 & \cellcolor{green!25}\textbf{0.88} & \cellcolor{yellow!50}\textbf{0.81} \\
  hopper-m-e & 0.60 & 0.57 & 0.69 & 0.63 & \cellcolor{green!25}\textbf{0.75} & \cellcolor{yellow!50}\textbf{0.68} \\
  walker2d-m-e & 0.52 & 0.46 & 0.67 & 0.59 & \cellcolor{green!25}\textbf{0.84} & \cellcolor{yellow!50}\textbf{0.80} \\ \hline
  \textbf{Mean} & 0.56 & 0.50 & 0.63 & 0.57 & \cellcolor{green!25}\textbf{0.83} & \cellcolor{yellow!50}\textbf{0.76} \\
  \bottomrule
  
\end{tabular}

\end{table}

\noindent \textbf{The Choice of Search Vector:} Recall that in Equation (\ref{eq:2}), we use $(s\oplus a \oplus s^\prime)$ as the search vector for KNN search because this reflects the difference between simulated dynamics and real dynamics. What if we replace the search vector with $(s\oplus s^\prime)$ or $(s\oplus a)$?

We design experiments to investigate the effectiveness of these three choices of search vectors. We follow the experimental settings in Section~\ref{sec:experiment} and conduct experiments with these search vectors on 6 D4RL MuJoCo datasets. We also use $\rho$ and $r$ as evaluation metrics. We present the experimental results in Table~\ref{tab:3}. It is evident that using $(s\oplus a\oplus s^\prime)$ as the search vector is superior to using $(s\oplus s^\prime)$, or $(s\oplus a)$. The scores in terms of $\rho$ and $r$ on all six datasets are the highest when using $(s\oplus a\oplus s^\prime)$ among the three choices. It turns out that it is vital to include the action $a$ for measuring uncertainty. The inclusion of the next state is also necessary because the learned dynamics model is responsible for predicting the next state given $(s,a)$.

\noindent \textbf{The Choice of Distance Measure:} Different distance measure for KNN search can induce different uncertainty estimation. To examine whether SUMO is sensitive to different distance measure, we change the default Euclidean distance we use to Manhattan distance and Cosine similarity, and conduct experiments on six MuJoCo datasets, based on MOPO+SUMO. The results are presented in Table~\ref{tab:4}. We can observe using Euclidean distance is slightly better than other two distance measure, but there is no obvious performance distinction, so we can simply use Euclidean distance as distance measure.

\begin{table}
\caption{Performance comparison of MOPO+SUMO between difference distance measure. Each experiment is run with 5 seeds.}
\label{tab:4}
\centering
\begin{tabular}{c|c|c|c}
\toprule
\textbf{Task Name}  & Euclidean  & Manhattan & Cosine \\ 
\midrule
half-m  & \cellcolor{green!25}\textbf{68.9} & 67.4 & 68.0\\
hopper-m  & \cellcolor{green!25}\textbf{74.6} & 72.9 & 73.8 \\
walker2d-m  & 57.3 & \cellcolor{green!25}\textbf{59.1} & 56.2 \\
half-m-e  & \cellcolor{green!25}\textbf{84.1} & 83.9 & 82.1 \\
hopper-m-e  & 88.1 & \cellcolor{green!25}\textbf{90.4} & 88.6 \\
walker2d-m-e  & \cellcolor{green!25}\textbf{81.9} & 79.4 & 77.2 \\
\midrule
\textbf{Average Score} & \cellcolor{green!25}\textbf{75.8} & 75.5 & 74.3\\
\bottomrule
\end{tabular}
\end{table}

% We then investigate the hyperparameter sensitivity of SUMO. For SUMO, the main hyperparameter is the number of nearest neighbors $k$ in KNN search.

\begin{table}
\caption{Spearman rank ($\rho$) and Pearson bivariate ($r$) correlations comparison of SUMO with different values of $k$. We run each experiment by 5 different seeds and report the average $\rho$ and $r$.}
\label{tab:5}
    \centering
\begin{tabular}{l | cc | cc | cc}
\toprule
\multicolumn{1}{c}{\multirow{2}{*}{\textbf{Task Name}}} &
  \multicolumn{2}{c}{$k=1$} &
  \multicolumn{2}{c}{$k=5$} & 
  \multicolumn{2}{c}{$k=10$} 
  \\ \cline{2-7} 
\multicolumn{1}{c}{} &
  $\mathbf{\rho}$ &
  $r$ &
  $\mathbf{\rho}$ &
  $r$ &
  $\mathbf{\rho}$ &
  $r$ \\ \hline
  half-m & 0.84 & 0.77 & 0.81 & 0.75 & \cellcolor{green!25}\textbf{0.86} & \cellcolor{yellow!50}\textbf{0.78} \\
  hopper-m & \cellcolor{green!25}\textbf{0.86} & 0.79 & 0.84 & \cellcolor{yellow!50}\textbf{0.81} & 0.84 & 0.80 \\
  walker2d-m & 0.82 & 0.71 & \cellcolor{green!25}\textbf{0.85} & \cellcolor{yellow!50}\textbf{0.73} & 0.81 & 0.70 \\
  half-m-e & \cellcolor{green!25}\textbf{0.88} & \cellcolor{yellow!50}\textbf{0.81} & 0.86 & 0.77 & 0.84 & 0.79 \\
  hopper-m-e & 0.75 & 0.68 & \cellcolor{green!25}\textbf{0.77} & 0.71 & 0.74 & \cellcolor{yellow!50}\textbf{0.73} \\
  walker2d-m-e & \cellcolor{green!25}\textbf{0.84} & \cellcolor{yellow!50}\textbf{0.80} & 0.82 & 0.76 & 0.80 & 0.77 \\ \hline
  \textbf{Mean} & \cellcolor{green!25}\textbf{0.83} & 0.76 & 0.82 & 0.76 & 0.81 & \cellcolor{yellow!50}\textbf{0.77} \\
  \bottomrule

\end{tabular}

\end{table}

\noindent \textbf{Number of Neighbors $k$:} The main hyperparameter in SUMO is the number of nearest neighbors $k$ in KNN search. To examine its influence on the uncertainty estimation, we sweep $k$ across $\{1,5,10\}$ and run experiments on selected MuJoCo datasets. We use metrics of $\rho$ and $r$. The results are shown in Table~\ref{tab:5}. We find that SUMO is robust to $k$.

\section{Conclusion}
In this paper, we propose SUMO, a novel search-based uncertainty estimation method for model-based offline RL. SUMO characterizes the uncertainty as the cross entropy between simulated dynamics and true dynamics, and employ a practical KNN search method for implementation. We show that SUMO can provide a better uncertainty estimation than model ensemble-based methods and boost the performance of a variety of base algorithms. One major limitation of SUMO is the heavy computational overhead for KNN search with increasing data dimension and dataset size. A possible solution for this can be using a distance-preserving randomized neural network~\cite{zhang2023darl,van2021feature} to map the original space to a low-dimensional hidden space, and we leave it for future work.

\bibliography{aaai25}

\clearpage
\section{Reproducibility Checklist}
\begin{enumerate}

\item {\textbf{Clarification}}
    \item[] This paper:
    \begin{itemize}
        \item Includes a conceptual outline and/or pseudocode description of AI methods introduced \qquad \textcolor{purple}{Yes}
        \item Clearly delineates statements that are opinions, hypothesis, and speculation from objective facts and results \qquad \textcolor{purple}{Yes}
        \item Provides well marked pedagogical references for less-familiare readers to gain background necessary to replicate the paper \qquad \textcolor{purple}{Yes}
    \end{itemize}

\item {\textbf{Theoretical contributions}}
    \item[] Does this paper make theoretical contributions? \qquad \textcolor{purple}{Yes}
    \item[] If yes, please complete the list below.
    \begin{itemize}
        \item All assumptions and restrictions are stated clearly and formally. \qquad \textcolor{purple}{Yes}
        \item All novel claims are stated formally (e.g., in theorem statements). \qquad \textcolor{purple}{Yes}
        \item Proofs of all novel claims are included. \qquad \textcolor{purple}{Yes}
        \item Proof sketches or intuitions are given for complex and/or novel results. \qquad \textcolor{purple}{Yes}
        \item Appropriate citations to theoretical tools used are given. \qquad \textcolor{purple}{Yes}
        \item All theoretical claims are demonstrated empirically to hold. \qquad \textcolor{purple}{NA}
        \item All experimental code used to eliminate or disprove claims is included. \qquad \textcolor{purple}{NA} 
    \end{itemize}

\item {\textbf{Datasets}}
    \item[] Does this paper rely on one or more datasets? \qquad \textcolor{purple}{Yes}
    \item[] If yes, please complete the list below.
    \begin{itemize}
        \item A motivation is given for why the experiments are conducted on the selected datasets. \qquad \textcolor{purple}{Yes}
        \item All novel datasets introduced in this paper are included in a data appendix. \qquad \textcolor{purple}{NA}
        \item All novel datasets introduced in this paper will be made publicly available upon publication of the paper with a license that allows free usage for research purposes. \qquad \textcolor{purple}{NA}
        \item All datasets drawn from the existing literature (potentially including authors’ own previously published work) are accompanied by appropriate citations. \qquad \textcolor{purple}{Yes}
        \item All datasets drawn from the existing literature (potentially including authors’ own previously published work) are publicly available. \qquad \textcolor{purple}{Yes}
        \item All datasets that are not publicly available are described in detail, with explanation why publicly available alternatives are not scientifically satisficing. \qquad \textcolor{purple}{NA}
    \end{itemize}
\item {\textbf{Experimental setting}}
    \item[] Does this paper include computational experiments? \qquad \textcolor{purple}{Yes}
    \item[] If yes, complete the list below.
    \begin{itemize}
        \item Any code required for pre-processing data is included in the appendix. \qquad \textcolor{purple}{Yes}
        \item All source code required for conducting and analyzing the experiments is included in a code appendix. \qquad \textcolor{purple}{Yes}
        \item All source code required for conducting and analyzing the experiments will be made publicly available upon publication of the paper with a license that allows free usage for research purposes. \qquad \textcolor{purple}{Yes}
        \item All source code implementing new methods have comments detailing the implementation, with references to the paper where each step comes from \qquad \textcolor{purple}{Yes}
        \item If an algorithm depends on randomness, then the method used for setting seeds is described in a way sufficient to allow replication of results. \qquad \textcolor{purple}{Yes}
        \item This paper specifies the computing infrastructure used for running experiments (hardware and software), including GPU/CPU models; amount of memory; operating system; names and versions of relevant software libraries and frameworks. \qquad \textcolor{purple}{Yes}
        \item  This paper formally describes evaluation metrics used and explains the motivation for choosing these metrics. \qquad \textcolor{purple}{Yes}
        \item This paper states the number of algorithm runs used to compute each reported result. \qquad \textcolor{purple}{Yes}
        \item Analysis of experiments goes beyond single-dimensional summaries of performance (e.g., average; median) to include measures of variation, confidence, or other distributional information. \qquad \textcolor{purple}{Yes}
        \item The significance of any improvement or decrease in performance is judged using appropriate statistical tests (e.g., Wilcoxon signed-rank). \qquad \textcolor{purple}{Yes}
        \item This paper lists all final (hyper-)parameters used for each model/algorithm in the paper’s experiments. \qquad \textcolor{purple}{Yes}
        \item This paper states the number and range of values tried per (hyper-) parameter during development of the paper, along with the criterion used for selecting the final parameter setting. \qquad \textcolor{purple}{Yes}
    \end{itemize}
\end{enumerate}
\clearpage
\appendix

\section{A. Missing Proofs}
\label{sec:missingproofs}
\subsection{A.1. Proof of Theorem \ref{theorem:moposumo}}

To prove theorem \ref{theorem:moposumo}, we introduce the conclusion which is taken directly from the MOPO paper, with some extra necessary backgrounds and definitions.

\begin{lemma}
    \label{theo:mopo}
    Let $\mathcal{F}$ be the set of functions mapping from $S$ to $\mathbb{R}$ that contain the value function $V^\pi_{\mathcal{M}}:=\mathbb{E}_\pi[\sum_{t=0}^\infty\gamma^t r(s_t,a_t)|s_0\sim\rho]$. We denote $d_{\mathcal{F}}$ as the integral probability metric (IPM) defined by $\mathcal{F}$. Suppose that for all $\pi$, $V_\mathcal{M}^\pi\in c\mathcal{F}$, where $c\in\mathbb{R}$ is a constant. We say $u: S\times A\rightarrow \mathbb{R}$ is an admissible error estimator for $\widehat{\mathcal{M}}$ if $d_{\mathcal{F}}(P_{\widehat{\mathcal{M}}}(\cdot|s,a), P(\cdot|s,a))\le u(s,a), \forall\, s,a$. Denote $\lambda:= \gamma c$, and $\epsilon_u(\pi) = \mathbb{E}_{\hat{\rho}^\pi}[u(s,a)]$. Then the learned policy given by MOPO satisfies,
    \begin{equation}
        J_\rho(\hat{\pi}) \ge J_\rho(\pi) - 2\lambda \epsilon_u(\pi).
    \end{equation}
\end{lemma}

\begin{proof}
    Please see the proof of Theorem 4.4 in the MOPO paper for details.
\end{proof}

Now, we present the formal Theorem \ref{theorem:moposumo} here.

\begin{theorem}[Theorem \ref{theorem:moposumo} Formal]
    \label{theo:appmoposumo}
    Let $\mathcal{F}$ be the set of functions mapping from $S$ to $\mathbb{R}$ that contain the value function $V^\pi_{\mathcal{M}}:=\mathbb{E}_\pi[\sum_{t=0}^\infty\gamma^t r(s_t,a_t)|s_0\sim\rho]$. We denote $d_{\mathcal{F}}$ as the integral probability metric (IPM) defined by $\mathcal{F}$. Suppose that for all $\pi$, $V_\mathcal{M}^\pi\in c\mathcal{F}$, where $c\in\mathbb{R}$ is a constant. We say $u: S\times A\rightarrow \mathbb{R}$ is an admissible error estimator for $\widehat{\mathcal{M}}$ if $d_{\mathcal{F}}(P_{\widehat{\mathcal{M}}}(\cdot|s,a), P(\cdot|s,a))\le u(s,a), \forall\, s,a$. Denote $\lambda:= \gamma c$, the behavior policy of the dataset as $\mu$, the SUMO uncertainty estimator $u(s,a,s^\prime)$ defined in Equation (\ref{eq:2}), then MOPO+SUMO incurs the policy $\pi$ that satisfies:
    \begin{equation*}
        J_\rho(\pi) \ge J_{\rho}(\mu) - 2\lambda \mathbb{E}_{\hat{\rho}^\pi}[u(s,a,s^\prime)],
    \end{equation*}
\end{theorem}

\begin{proof}
    Note that Theorem \ref{theo:mopo} is applicable to any reasonable uncertainty estimator $u(s,a)$, i.e.,
    \begin{equation}
    \label{eq:usa}
        J_\rho(\pi) \ge J_{\rho}(\mu) - 2\lambda \mathbb{E}_{\hat{\rho}^\pi}[u(s,a)],
    \end{equation}
    We choose $u(s,a)$ to be the following form:
    \begin{equation}
        u(s,a) = \log\left( \| (\hat{s}\oplus \hat{a}) - (\hat{s}\oplus \hat{a})^{k,N} \|_2 + 1 \right).
    \end{equation}
    Compared to Equation (\ref{eq:2}), no next state information is included in it. Now we use a simple example to derive an important conclusion. Suppose that $(s_0,a_0,s^\prime_0) = (\hat{s}\oplus \hat{a}\oplus \hat{s}^\prime)^{1,N}$, i.e., we set $k=1$ here. We also suppose that the nearest neighbor of $(\hat{s}\oplus \hat{a})$ is $(s_1,a_1)$, i.e., $(s_1,a_1) = (\hat{s}\oplus \hat{a})^{1,N}$. It is not difficult to find that $ \|(\hat{s}\oplus \hat{a}\oplus \hat{s}^\prime) - (\hat{s}\oplus \hat{a}\oplus \hat{s}^\prime)^{1,N} \|_2 \ge \|(\hat{s}\oplus \hat{a}) - (\hat{s}\oplus \hat{a})^{1,N}\|$ because,
    \begin{equation}
    \begin{aligned}
        &\| (\hat{s}\oplus \hat{a} \oplus\hat{s}^\prime) - (\hat{s}\oplus \hat{a}\oplus \hat{s}^\prime)^{1,N} \|_2 \\
        &\qquad\qquad = \sqrt{(\hat{s} - s_0)^2 + (\hat{a} - a_0)^2 + (\hat{s}^\prime - s^\prime_0)^2} \\
        &\qquad\qquad \ge \sqrt{(\hat{s} - s_0)^2 + (\hat{a} - a_0)^2} \\
        &\qquad\qquad \ge \sqrt{(\hat{s} - s_1)^2 + (\hat{a} - a_1)^2} \\
        &\qquad\qquad = \| (\hat{s}\oplus \hat{a}) - (\hat{s}\oplus \hat{a})^{1,N} \|_2,
    \end{aligned}
    \end{equation}
    where the last inequality holds because $(s_1,a_1)$ is the nearest neighbor of $(\hat{s},\hat{a})$. That being said, $\sqrt{(\hat{s} - s_1)^2 + (\hat{a} - a_1)^2}$ should be the smaller than $\sqrt{(\hat{s} - s_0)^2 + (\hat{a} - a_0)^2}$, otherwise $(s_0,a_0)$ will become the nearest neighbor. This conclusion also holds for $k\ge 2$ by using induction and following a similar way of the above proof. We then have
    \begin{equation}
        \| (\hat{s}\oplus \hat{a}) - (\hat{s}\oplus \hat{a})^{k,N} \|_2 \le \| (\hat{s}\oplus \hat{a} \oplus\hat{s}^\prime) - (\hat{s}\oplus \hat{a}\oplus \hat{s}^\prime)^{k,N} \|_2.
    \end{equation}
    % Compared to Equation (\ref{eq:2}), no next state information is included in it. Now we use a simple example to derive an important conclusion. Suppose that $(\hat{a},\hat{b}) = (a\oplus b)^{1,N}$ and $(\hat{a},\hat{b}, \hat{c}) = (a\oplus b \oplus c)^{1,N}$ (i.e., $k=1$). It is not difficult to find that $\forall a,b,c, \|(a\oplus b) - (a\oplus b)^{1,N}\|_2 \le \|(a\oplus b \oplus c) - (a\oplus b \oplus c)^{1,N}\|$ because $\sqrt{(a-\hat{a})^2 + (b - \hat{b})^2 + (c - \hat{c})^2} \ge \sqrt{(a - \hat{a})^2 + (b - \hat{b})^2}$. This conclusion also holds for $k\ge 2$. We then have
    % \begin{equation}
    %     \| (\hat{s}\oplus \hat{a}) - (\hat{s}\oplus \hat{a})^{k,N} \|_2 \le \| (\hat{s}\oplus \hat{a} \oplus\hat{s}^\prime) - (\hat{s}\oplus \hat{a}\oplus \hat{s}^\prime)^{k,N} \|_2.
    % \end{equation}
    Then, it is natural to have
    \begin{align*}
        u(s,a) &= \log\left( \| (\hat{s}\oplus \hat{a}) - (\hat{s}\oplus \hat{a})^{k,N} \|_2 + 1 \right) \\
        &\le \log\left( \| (\hat{s}\oplus \hat{a} \oplus\hat{s}^\prime) - (\hat{s}\oplus \hat{a}\oplus \hat{s}^\prime)^{k,N} \|_2 + 1 \right) \\
        &= u(s,a,s^\prime).
    \end{align*}
    Combining the above results into Equation (\ref{eq:usa}), we have
    \begin{equation}
        J_\rho(\pi) \ge J_{\rho}(\mu) - 2\lambda \mathbb{E}_{\hat{\rho}^\pi}[u(s,a)] \ge J_{\rho}(\mu) - 2\lambda \mathbb{E}_{\hat{\rho}^\pi}[u(s,a,s^\prime)],
    \end{equation}
    This concludes the proof.
\end{proof}

\subsection{A.2. Proof of Theorem~\ref{theorem:1}}
To prove theorem~\ref{theorem:1}, we introduce the following lemmas.

\begin{lemma}[Pinsker's Inequality]
\label{lemma:1}
If $P$ and $Q$ are two probablity distributions on a measurable space $(X,\Sigma)$, then we have:
\begin{equation*}
    D_{TV}(P,Q)\leq\sqrt{\frac{D_{KL}(P\|Q)}{2}}
\end{equation*}
where $D_{TV}(P,Q)$ is the total variance distance between $P$ and $Q$, and $D_{KL}(P\|Q)$ is the Kullback–Leibler divergence of $P$ from $Q$.
\end{lemma}

\begin{proof}
    Please refer to \citep{csiszar2011information} for the proof.
\end{proof}

\begin{lemma}
\label{lemma:2}
For $P$ and $Q$ in Lemma~\ref{lemma:1}, we have:
\begin{align*}
    D_{TV}(P,Q)\leq\sqrt{\frac{\mathcal{H}(P,Q)}{2}}
\end{align*}
where $\mathcal{H}(P,Q)$ is the cross entropy between $P$ and $Q$.
\end{lemma}
\begin{proof}
    Notice that $D_{KL}(P\|Q)=\mathcal{H}(P,Q)-\mathcal{H}(P)$, using Lemma~\ref{lemma:1}, we have:
        \begin{align*}
            D_{TV}(P,Q)&\leq\sqrt{\frac{D_{KL}(P\|Q)}{2}}\\
            &\leq\sqrt{\frac{\mathcal{H}(P,Q)-\mathcal{H}(P)}{2}}\\
            &\leq\sqrt{\frac{\mathcal{H}(P,Q)}{2}}
        \end{align*}
\end{proof}

\begin{lemma}
    \label{lemma:3}
    Given the $\epsilon$-Model MDP $\widehat{\mathcal{M}}_\epsilon$ and the dataset MDP $\mathcal{M}_d$, the return difference for any policy $\pi$ in $\widehat{\mathcal{M}}_\epsilon$ and $\mathcal{M}_d$ satisfies:
    \begin{equation*}
    \left|J_{\hat{\rho}_\epsilon}(\pi,\widehat{\mathcal{M}_\epsilon})-J_{\rho_d}(\pi,\mathcal{M}_d)\right|\leq \frac{r_{max}}{1-\gamma}\left(\sqrt{\frac{\epsilon}{2}}+2D_{TV}(\hat{\rho}_\epsilon,\rho_d)\right)
\end{equation*}
\end{lemma}

\begin{proof}
    It's easy to have:
        \begin{align*}
            &\left|J_{\hat{\rho}_\epsilon}(\pi,\widehat{\mathcal{M}}_\epsilon)-J_{\rho_d}(\pi,\mathcal{M}_d)\right| \\
            &=\left|J_{\hat{\rho}_\epsilon}(\pi,\widehat{\mathcal{M}}_\epsilon)-J_{\hat{\rho}_\epsilon}(\pi,\mathcal{M}_d)+J_{\hat{\rho}_\epsilon}(\pi,\mathcal{M}_d)-J_{\rho_d}(\pi,\mathcal{M}_d)\right|\\
            &\leq \underbrace{\left|J_{\hat{\rho}_\epsilon}(\pi,\widehat{\mathcal{M}}_\epsilon)-J_{\hat{\rho}_\epsilon}(\pi,\mathcal{M}_d)\right|}_{L_1}+\underbrace{\left|J_{\hat{\rho}_\epsilon}(\pi,\mathcal{M}_d)-J_{\rho_d}(\pi,\mathcal{M}_d)\right|}_{L_2}
        \end{align*}
For $L_{1}$, we have:
\begin{align*}
&L_1=\left|J_{\hat{\rho}_\epsilon}(\pi,\widehat{\mathcal{M}}_\epsilon)-J_{\hat{\rho}_\epsilon}(\pi,\mathcal{M}_d)\right|\\
&=\left|\mathbb{E}_{\hat{\rho}_\epsilon}\mathbb{E}_\pi\mathbb{E}_{\hat{P}_\epsilon}\left[\sum_{t=0}^\infty\gamma^tr(s_t,a_t)\right]-\mathbb{E}_{\hat{\rho}_\epsilon}\mathbb{E}_\pi\mathbb{E}_{P_d}\left[\sum_{t=0}^\infty\gamma^tr(s_t,a_t)\right]\right|\\
&=\left|\sum\hat{\rho}_\epsilon\sum_t\sum_{a_t}\pi(a_t|s_t)\left(\hat{P}_\epsilon(\cdot|s_t,a_t)-P_d(\cdot|s_t,a_t)\right)\gamma^tr(s_t,a_t)\right|\\
&\leq r_{max}\cdot\left|\sum\hat{\rho}_\epsilon\right|\cdot\left|\sum_t\sum_{a_t}\pi(a_t|s_t)\left|\hat{P}_\epsilon(\cdot|s_t,a_t)-P_d(\cdot|s_t,a_t)\right|\gamma^t\right|\\
&\leq r_{max}\cdot\left|\sum_t\sum_{a_t}\pi(a_t|s_t)D_{TV}\left(\hat{P}_\epsilon(\cdot|s_t,a_t),P_d(\cdot|s_t,a_t)\right)\gamma^t\right|
\end{align*}
Notice that in the $\epsilon$-Model MDP, we require that for any $(s,a)$, $\mathcal{H}\left(\hat{P}_\epsilon(\cdot|s,a),P_d(\cdot|s,a)\right)\leq\epsilon$. Using Lemma~\ref{lemma:2}, we have $D_{TV}\left(\hat{P}_\epsilon(\cdot|s,a),P_d(\cdot|s,a)\right)\leq\sqrt{\frac{\epsilon}{2}}. $Therefore, we have:
\begin{align*}
    L_1&\leq r_{max}\cdot\sqrt{\frac{\epsilon}{2}}\cdot\left|\sum_t\sum_{a_t}\pi(a_t|s_t)\gamma^t\right|\\
    &=\frac{r_{max}}{1-\gamma}\cdot\sqrt{\frac{\epsilon}{2}}
\end{align*}

For $L_2$, we have:

\begin{align*}
    &L_2=\left|J_{\hat{\rho}_\epsilon}(\pi,\mathcal{M}_d)-J_{\rho_d}(\pi,\mathcal{M}_d)\right|\\
    &=\left|\mathbb{E}_{\hat{\rho}_\epsilon}\mathbb{E}_{\pi}\mathbb{E}_{P_d}\left[\sum_{t=0}^\infty\gamma^tr(s_t,a_t)\right]-\mathbb{E}_{\rho_d}\mathbb{E}_{\pi}\mathbb{E}_{P_d}\left[\sum_{t=0}^\infty\gamma^tr(s_t,a_t)\right]\right|\\
    &=\left|\sum(\hat{\rho}_\epsilon-\rho_d)\sum_t\sum_{a_t}\pi(a_t|s_t)P_d(\cdot|s_t,a_t)\gamma^tr(s_t,a_t)\right|\\
    &\leq r_{max}\cdot\left|\sum(\hat{\rho_\epsilon}-\rho_d)\right|\cdot\left|\sum_t\sum_{a_t}\pi(a_t|s_t)\gamma^t\right|\\
    &\leq \frac{r_{max}}{1-\gamma}\cdot\left(\sum\left|\hat{\rho}_\epsilon-\rho_d\right|\right)\\
    &=\frac{2r_{max}}{1-\gamma}\cdot D_{TV}(\hat{\rho}_\epsilon,\rho_d)
\end{align*}
Given the upper bounds of $L_1$ and $L_2$, it's easy to have:
\begin{equation*}
    \left|J_{\hat{\rho}_\epsilon}(\pi,\hat{\mathcal{M}_\epsilon})-J_{\rho_d}(\pi,\mathcal{M}_d)\right|\leq \frac{r_{max}}{1-\gamma}\left(\sqrt{\frac{\epsilon}{2}}+2D_{TV}(\hat{\rho}_\epsilon,\rho_d)\right)
\end{equation*}
Then we conclude the proof.
\end{proof}

\begin{lemma}
\label{lemma:4}
    Given the dataset MDP $\mathcal{M}_d$ and the true MDP $\mathcal{M}$, the return difference for any policy $\pi$ in $\mathcal{M}_d$ and $\mathcal{M}$ satisfies:
    \begin{align*}
J_{\rho_d}(\pi,\mathcal{M}_d)-J_\rho(\pi,\mathcal{M})&\geq         -\frac{r_{max}}{1-\gamma}\left(1+2D_{TV}(\rho_d,\rho)\right)\\
J_{\rho_d}(\pi,\mathcal{M}_d)-J_\rho(\pi,\mathcal{M})&\leq \frac{2r_{max}}{1-\gamma}D_{TV}(\rho_d,\rho)
    \end{align*}
    
\end{lemma}

\begin{proof}
    It's easy to have:
    \begin{align*}
        &J_{\rho_d}(\pi,\mathcal{M}_d)-J_\rho(\pi,\mathcal{M})\\
        =&\underbrace{J_{\rho_d}(\pi,\mathcal{M}_d)-J_{\rho_d}(\pi,\mathcal{M})}_{L_1}+\underbrace{J_{\rho_d}(\pi,\mathcal{M})-J_\rho(\pi,\mathcal{M})}_{L_2}
    \end{align*}

For $L_1$, we have:
\begin{align*}
    \quad L_1&=J_{\rho_d}(\pi,\mathcal{M}_d)-J_{\rho_d}(\pi,\mathcal{M})\\
    &=\mathbb{E}_{\rho_d}\mathbb{E}_{\pi}\mathbb{E}_{P_d}\left[\sum_{t=0}^\infty\gamma^tr(s_t,a_t)\right]-\mathbb{E}_{\rho_d}\mathbb{E}_{\pi}\mathbb{E}_{P}\left[\sum_{t=0}^\infty\gamma^tr(s_t,a_t)\right]\\
    &=\sum\rho_d\sum_t\sum_{a_t}\pi(a_t|s_t)\left(P_d(\cdot|s_t,a_t)-P(\cdot|s_t,a_t)\right)\gamma^tr(s_t,a_t)
\end{align*}

For transitions $(s,a,r,s^\prime)$ within the dataset, we have $P_d(\cdot|s,a)-P(\cdot|s,a)=0$. For those not present in the datset, we have $P_d(\cdot|s,a)=0$, so $P_d(\cdot|s,a)-P(\cdot|s,a)<0$. As a result, we can have $P_d(\cdot|s,a)-P(\cdot|s,a)\leq0$. Also, considering $r(s,a)\geq 0$, it's easy to have $L_1\leq 0$.

To obtain the lower bound for $L_1$, we have:
\begin{align*}
    \left|L_1\right|&=\left|\sum\rho_d\sum_t\sum_{a_t}\pi(a_t|s_t)\left(P_d(\cdot|s_t,a_t)-P(\cdot|s_t,a_t)\right)\gamma^tr(s_t,a_t)\right|\\
    &\leq r_{max}\cdot\left|\sum\rho_d\right|\cdot\left|\sum_t\sum_{a_t}\left|P(\cdot|s_t,a_t)\right|\gamma^t\right|\\
    &\leq \frac{r_{max}}{1-\gamma}
\end{align*}
So we can have $L_1\geq -\frac{r_{max}}{1-\gamma}$. As a result, we have $-\frac{r_{max}}{1-\gamma}\leq L_1 \leq 0$.

For $L_2$, we have:
\begin{align*}
    \left|L_2\right|&=\left|J_{\rho_d}(\pi,\mathcal{M})-J_\rho(\pi,\mathcal{M})\right|\\
    &=\left|\mathbb{E}_{\rho_d}\mathbb{E}_\pi\mathbb{E}_{P}\left[\sum_{t=0}^\infty\gamma^tr(s_t,a_t)\right]-\mathbb{E}_{\rho}\mathbb{E}_\pi\mathbb{E}_{P}\left[\sum_{t=0}^\infty\gamma^tr(s_t,a_t)\right]\right|\\
    &=\left|\sum(\rho_d-\rho)\sum_t\sum_{a_t}\pi(a_t|s_t)P(\cdot|s_t,a_t)\gamma^tr(s_t,a_t)\right|\\
    &\leq r_{max}\cdot\left|\sum(\rho_d-\rho)\right|\cdot\left|\sum_t\sum_{a_t}\pi(a_t|s_t)\gamma^t\right|\\
    &\leq\frac{2r_{max}}{1-\gamma}D_{TV}(\rho_d,\rho)
\end{align*}

Given the bounds of $L_1$ and $L_2$, we can have the desired results. 

\end{proof}

Then we can prove Theorem \ref{theorem:1}. We restate it as follows:
\begin{theorem}[Performance bounds]
    For any policy $\pi$, the return of $\pi$ in $\epsilon$-Model MDP $\widehat{\mathcal{M}}_\epsilon$ and the original MDP $\mathcal{M}$ satisfies:
    \begin{equation*}
        \begin{split}
            J_{\hat{\rho}_\epsilon}(\pi,\widehat{\mathcal{M}}_\epsilon) \geq &J_{\rho}(\pi,\mathcal{M})-\frac{r_{max}}{1-\gamma}( 1+\sqrt{\frac{\epsilon}{2}}+2D_{TV}(\rho,\rho_d)  \\
            &+2D_{TV}(\hat{\rho}_\epsilon,\rho_d))
        \end{split}
    \end{equation*}

    \begin{equation*}
        \begin{split}
            J_{\hat{\rho}_\epsilon}(\pi,\widehat{\mathcal{M}}_\epsilon) \leq &J_{\rho}(\pi,\mathcal{M})+\frac{r_{max}}{1-\gamma}(\sqrt{\frac{\epsilon}{2}}+2D_{TV}(\rho,\rho_d)  \\
            &+2D_{TV}(\hat{\rho}_\epsilon,\rho_d))
        \end{split}
    \end{equation*}
\end{theorem}

\begin{proof}
    It's easy to notice:
    
\begin{align*}
J_{\hat{\rho}_\epsilon}(\pi,\widehat{\mathcal{M}}_\epsilon)-J_{\rho}(\pi,\mathcal{M})&=\underbrace{J_{\hat{\rho}_\epsilon}(\pi,\widehat{\mathcal{M}}_\epsilon)-J_{\rho_d}(\pi,\mathcal{M}_d)}_{L_1}\\
&+\underbrace{J_{\rho_d}(\pi,\mathcal{M}_d)-J_{\rho}(\pi,\mathcal{M})}_{L_2}
\end{align*}

According to Lemma~\ref{lemma:3} and Lemma~\ref{lemma:4}, we have

\begin{align*}
 -\frac{r_{max}}{1-\gamma}\left(\sqrt{\frac{\epsilon}{2}}+2D_{TV}(\hat{\rho}_\epsilon,\rho_d)\right)\leq L_1&\leq \frac{r_{max}}{1-\gamma}(\sqrt{\frac{\epsilon}{2}}+ \\
 &2D_{TV}(\hat{\rho}_\epsilon,\rho_d))
\end{align*}

\begin{equation*}
     -\frac{r_{max}}{1-\gamma}\left(1+2D_{TV}(\rho_d,\rho)\right)\leq L_2 \leq \frac{2r_{max}}{1-\gamma}D_{TV}(\rho_d,\rho)
\end{equation*}

Then we can easily get the performance bounds, which conclude the proof.

\end{proof}

\section{B. Detailed Pesudo-codes}
\label{sec:appendixpesudocodes}
In this section, we provide the full pseudocodes for MOPO+SUMO and MOReL+SUMO in Algorithm \ref{alg:moposumoformal} and Algorithm~\ref{alg:amorelsumoformal}, respectively.

\begin{algorithm}[tb]
\caption{MOPO+SUMO}
\label{alg:moposumoformal}
\begin{algorithmic}[1] %[1] enables line numbers
\STATE \textbf{Require:} Offline dataset $\mathcal{D}$, number of epochs $N$, maximum trajectory length $H$, penalty coefficient $\lambda$, initial policy $\pi_\theta$
\STATE Initialize synthetic dataset $\mathcal{D}_{\rm model} \leftarrow \emptyset$
\STATE Train the ensemble dynamics model $\{\hat{P}_{\psi_i
}(s^\prime|s,a) = \mathcal{N}(\mu_{\psi_i}(s,a), \Sigma_{\psi_i}(s,a))\}_{i=1}^N$ on $\mathcal{D}$

\FOR{epoch from 1 to $N$}
\STATE Sample an initial state $s_0$ from dataset $\mathcal{D}$
\FOR{$h$ in 1 to $H$}
\STATE Sample an action $a_h\sim\pi_\theta(\cdot|s_h)$
\STATE Randomly pick an ensemble member $\hat{P}$ from $\{\hat{P}_{\psi_i}\}_{i=1}^N$ and sample next state $s_{h+1}\sim\hat{P}(s_{h+1}|s_h,a_h)$
\STATE Calculate the reward penalty $\hat{u}$ using Equation (\ref{eq:5})
\STATE Calculate the new reward as $r_h=r_h-\lambda\hat{u}$
\STATE Add the synthetic transition $(s_h,a_h,r_h,s_{h+1})$ to the synthetic dataset $\mathcal{D}_{model}$
\STATE Sample a batch of transitions from $\mathcal{D} \cup \mathcal{D}_{model}$ and optimize the policy $\pi_\theta$ via the base RL algorithm such as SAC
\ENDFOR
\ENDFOR
\end{algorithmic}
\end{algorithm}

% \noindent \textbf{Step1: Constrcting the Environmental Dynamics Model:} This step is just the same as the first step in AMOReL+SUMO, so we won't elaborate on it again. After this step, we obtain an ensemble dynamics model $\hat{P}_{\psi}=\{\hat{P}_{\psi_i}\}_{i=1}^N$.

% \noindent \textbf{Step2: Reward penalty calculation with SUMO:} We utilize the current policy $\pi$ and the dynamics model $\hat{P}_{\psi}$ to generate synthetic samples. For each generated sample $(s,a,s^\prime)$, we first calculate its unnormalized uncertainty with SUMO:
% \begin{equation}
%     u(s,a,s^\prime)=\log \left(\left\|(s\oplus a\oplus s^\prime)-(s\oplus a\oplus s^\prime)^{k,N} \right\|_2+1\right)
% \end{equation}
% Then we normalize $u(s,a,s^\prime)$ to the range $[0,r_{max}]$ by dividing $u(s,a,s^\prime)$ by the maximum value of all $u$ so far, and then multiplying by $r_{max}$:
% \begin{equation}
% \label{eq:5}
%     \epsilon(s,a,s^\prime) = \frac{u(s,a,s^\prime)}{\max_{(s,a,s^\prime)} u(s,a,s^\prime)} \times r_{max}
% \end{equation}
% We incorporate $\epsilon$ as the reward penalty and calculate the new reward as $r=r-\lambda\epsilon$, where $\lambda$ is a hyperparameter that controls the magnitude of the penalty. We can then add the synthetic samples with reward penalty to $\mathcal{D}_{model}$.

% \noindent \textbf{Step 3: Model-based Policy Optimization:} Once we have acquired a sufficient number of synthetic samples, we can proceed to train the algorithm using samples drawn from $\mathcal{D}$ and $\mathcal{D}_{model}$, similar to the third step in AMOReL+SUMO.

\begin{algorithm}[tb]
\caption{AMOReL+SUMO}
\label{alg:amorelsumoformal}
\begin{algorithmic}[1] %[1] enables line numbers
\STATE \textbf{Require:} Offline dataset $\mathcal{D}$, number of epochs $N$, maximum trajectory length $H$, initial policy $\pi_\theta$
\STATE Initialize synthetic dataset $\mathcal{D}_{\rm model} \leftarrow \emptyset$
\STATE Train the ensemble dynamics model $\{\hat{P}_{\psi_i
}(s^\prime|s,a) = \mathcal{N}(\mu_{\psi_i}(s,a), \Sigma_{\psi_i}(s,a))\}_{i=1}^N$ on $\mathcal{D}$ using Equation (\ref{eq:3})
\STATE Calculate the truncating threshold $\epsilon$ using Equation (\ref{eq:4})
\FOR{epoch from 1 to $N$}
\STATE Sample an initial state $s_0$ from dataset $\mathcal{D}$
\FOR{$h$ in 1 to $H$}
\STATE Sample an action $a_h\sim\pi_\theta(\cdot|s_h)$
\STATE Randomly pick an ensemble member $\hat{P}$ from $\{\hat{P}_{\psi_i}\}_{i=1}^N$ and sample next state $s_{h+1}\sim\hat{P}(s_{h+1}|s_h,a_h)$
\STATE Calculate the sample uncertainty $u_h=\log\left(\left\|(s_h\oplus a_h\oplus s_{h+1})-(s_h\oplus a_h\oplus s_{h+1})^{k,N}\right\|_2+1\right)$ by FAISS
\IF{$u_h \le \epsilon$}
\STATE Add the synthetic transition $(s_h,a_h,r_h,s_{h+1})$ to the synthetic dataset $\mathcal{D}_{model}$
\STATE Sample a batch of transitions from $\mathcal{D}\cup\mathcal{D}_{model}$ and optimize the policy $\pi_\theta$ via the base RL algorithm such as SAC
\ELSE
\STATE break \qquad // Truncate the synthetic trajectory
\ENDIF
\ENDFOR
\ENDFOR
\end{algorithmic}
\end{algorithm}

\section{C. More Details on the Experimental Setup}

In this work, we primarily utilize MuJoCo datasets from D4RL to evaluate our method. Additionally, we also conduct experiments on Antmaze datasets. So we first introduce the two datasets used in our work.

\subsection{C.1. MuJoCo Datasets from D4RL}

\begin{figure}
    \centering
    \includegraphics[width=0.3\linewidth]{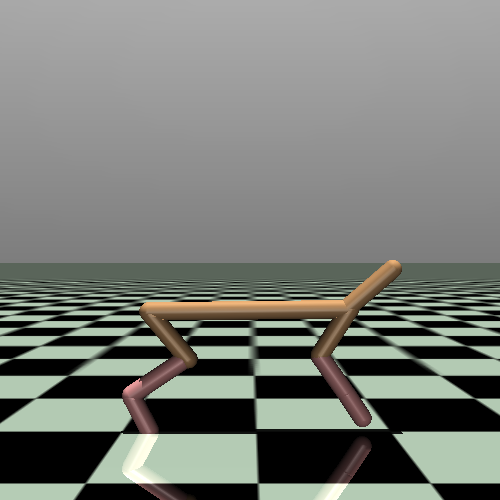}
    \includegraphics[width=0.3\linewidth]{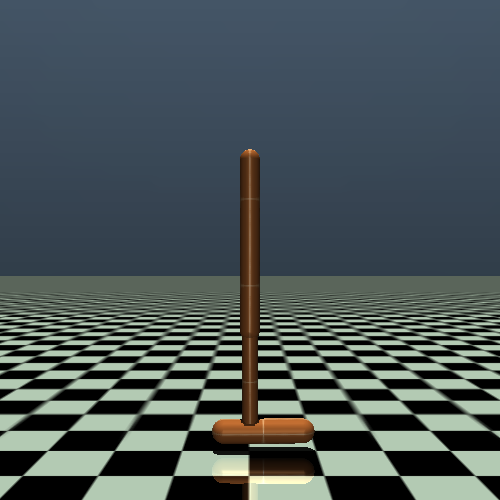}
    \includegraphics[width=0.3\linewidth]{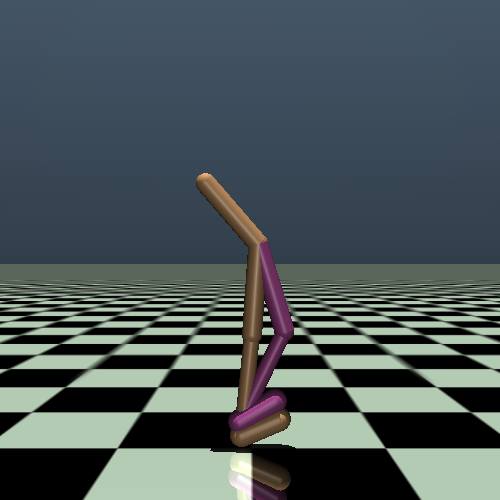}
    \caption{D4RL MuJoCo domains. From left to right, halfcheetah, hopper, walker2d.}
    \label{fig:mujocodataset}
\end{figure}

MuJoCo datasets from D4RL are collected from three tasks of MuJoCo: Halfcheetah, Hopper and Walker2d, as illustrated in Figure~\ref{fig:mujocodataset}. The quality of the datasets is categorized into the following types: \textit{random}: samples from a randomly initialized policy. \textit{medium}: 1M samples from an early-stopped SAC policy. \textit{medium-replay}: 1M samples gathered from the replay buffer of a policy trained up to the performance of the medium agent. \textit{medium-expert}: 50-50 split of medium-level data and expert-level data. \textit{expert}: 1M samples from a logged expert policy.

The metric we use to evaluate the agent's performance on MuJoCo datasets is the normalized score (NS). For a given MuJoCo dataset, NS is computed as:

\begin{equation*}
    NS=\frac{J_{\pi}-J_{random}}{J_{expert}-J_{random}}\times 100.
\end{equation*}
where $J_{\pi}$ is the performance of the policy under evaluation, $J_{random}$ is the performance of the random policy, and $J_{expert}$ is the performance of the expert policy.

\subsection{C.2. Antmaze Datasets from D4RL}

The Antmaze domain is a navigation task requiring an 8-DOF "Ant" quadraped robot to reach a goal location. Due to the sparse reward setting of Antmaze, it is usually more challenging than MuJoCo domain. The Antmaze datasets contain three maze layouts: "umaze", "medium" and "large", as shown in Figure~\ref{fig:antmaze}. The datasets are collected in three flavors: (1) the robot is commanded to reach a specified goal from a fixed start point (antmaze-umaze-v0). (2) the robot is commanded to reach a random goal from a random start point (the "diverse" datasets). (3) the robot is commanded to reach specific locations from a different set of specific start locations (the "play" datasets). In our experiment, we use the six datasets to evaluate our method: \textit{antmaze-umaze-v0}, \textit{antmaze-umaze-diverse-v0}, \textit{antmaze-medium-diverse-v0}, \textit{antmaze-medium-play-v0}, \textit{antmaze-large-diverse-v0}, \textit{antmaze-large-play-v0}. When evaluating the performance of a policy on the Antmaze datasets, we use the success rate of reaching the goal as the metric.

\begin{figure}[ht]
    \centering
    \includegraphics[width=\linewidth]{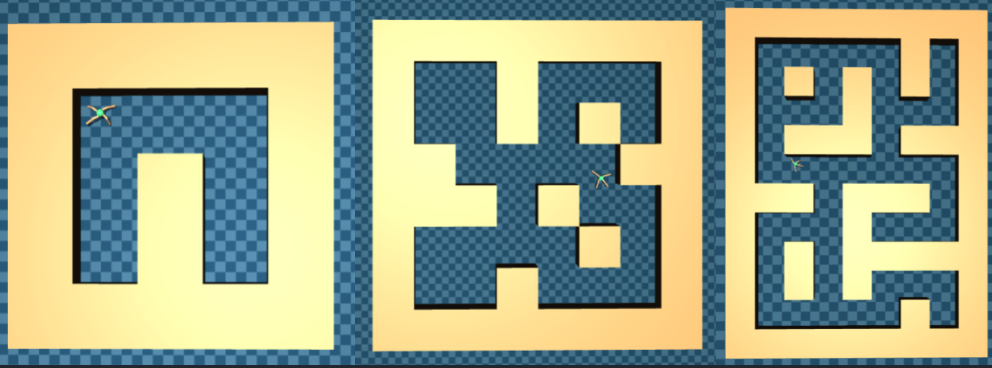}
    \caption{D4RL Antmaze domains. From left to right, umaze, medium, large.}
    \label{fig:antmaze}
\end{figure}

\subsection{C.3. Implementation Details}
\label{sec:implementationdetails}
We introduce the implementation of the baseline algorithms used in our work. For MOPO, we use the codebase\footnote{https://github.com/junming-yang/mopo.git} re-implemented in Pytorch. For MOReL, we use the author-provided\footnote{https://github.com/aravindr93/mjrl.git} implementation. For MOBILE\footnote{https://github.com/yihaosun1124/mobile} and RAMBO\footnote{https://github.com/marc-rigter/rambo.git}, we use the official implementation. For COMBO, we build the code based on MOPO. For the implementation of AMOReL and SUMO, we have included the code in the supplementary materials.

We then introduce the hyperparameter setup in our experiments. We present the main hyperparameter setup for AMOReL+SUMO and MOPO+SUMO in Table~\ref{tab:6}. For other baseline algorithms, we use the default hyperparameters suggested in original papers. 

\begin{table}[]
    \centering
    \caption{Hyperparameter setup for AMOReL+SUMO and MOPO+SUMO.}
    \begin{tabular}{lr}
    \toprule
    \textbf{Hyperparameter} & \textbf{Value} \\
    \midrule
     \textbf{SUMO} \\
    \;$k$ & 1 \\
    \;Search vector & $s\oplus a\oplus s^\prime$ \\
    \;Distance measure & Euclidean distance \\
    \midrule
    \textbf{AMOReL} \\
    \;Model ensemble size & 7 \\
    \;Model hidden layer & $(400,400,400,400)$ \\
    \;Actor and Critic hidden layer & $(256,256,256)$ \\
    \;Batch size & 256 \\
    \;Optimizer & Adam~\citep{kingma2014adam} \\
    \;Actor learning rate & $3\times 10^{-4}$ \\
    \;Critic learning rate & $3\times 10^{-4}$ \\
    \;$\alpha$ & 5 \\
    \;$\eta$ & 0.9 \\
    \midrule
    \textbf{MOPO} \\
    \;Model ensemble size & 7\\
    \;Model hidden layer & $(400,400,400,400)$ \\
    \;Actor and Critic hidden layer & $(256,256,256)$\\
    \;Batch size & 256 \\
    \;Optimizer & Adam \\
    \;Actor learning rate & $3\times 10^{-4}$\\
    \;Critic learning rate & $3\times 10^{-4}$\\
    \;$\lambda$ & 1 \\
    \;$\eta$ & 0.9 \\
    \bottomrule
    \end{tabular}

    \label{tab:6}
\end{table}

\subsection{C.4. Compute Infrastructure}
We list our hardware specifications as follows:
\begin{itemize}
    \item GPU: NVIDIA RTX 3090 ($\times$8)
    \item CPU: AMD EPYC 7452
\end{itemize}

\noindent We also list our software specifications as follows:
\begin{itemize}
    \item Python: 3.8.18
    \item Pytorch: 1.12.1+cu113
    \item Gym: 0.22.0
    \item MuJoCo: 2.0
    \item D4RL: 1.1
\end{itemize}

\begin{table}[]
    \centering
    \caption{Average training time comparison between base algorithms and their version with SUMO added on 15 MuJoCo datasets. The training steps are set to 1M.}
    \begin{tabular}{c|c|c}
    \toprule
        & MOPO & AMOReL \\
        \midrule
        Base & 7h23m & 8h46m \\
        +SUMO & \cellcolor{green!25}\textbf{6h41m} & \cellcolor{green!25}\textbf{8h03m} \\
    \bottomrule
    \end{tabular}
    \label{tab:timecost}
\end{table}

\subsection{C.5. Time Cost}
Given that SUMO involves KNN search which may be time consuming in large and high-dimensional datasets, it is essential to examine the time cost incurred by SUMO compared to model ensemble-based methods. We conduct experiments on all 15 MuJoCo datasets and compare the average training time cost of MOPO+SUMO and MOReL+SUMO against their base algorithms, with a total of 1M training steps. The results are shown in Table~\ref{tab:timecost}. The results indicate that SUMO does not bring more time cost than model ensemble-based methods, and even less, thanks to the efficient implementation of FAISS.

\section{D. More Experimental Results}

\subsection{D.1. Experimental Results on Antmaze Datasets}

\begin{table}
    \caption{Comparison of MOPO+SUMO and AMOReL+SUMO against MOPO and AMOReL on Antmaze "-v0" domains. We run each algorithm for 1M gradient steps with 5 random seeds.}
    \centering
    \footnotesize
    \begin{tabular}{c|cc|cc}
    \toprule
    \multicolumn{1}{c}{\multirow{2}{*}{\textbf{Task Name}}} & \multicolumn{2}{c}{MOPO} & \multicolumn{2}{c}{AMOReL}\\ 
    \cline{2-5} 
    \multicolumn{1}{c}{} & \multicolumn{1}{c}{SUMO} & \multicolumn{1}{c}{Base} & \multicolumn{1}{c}{SUMO} & \multicolumn{1}{c}{Base} \\
    \midrule
    Umaze & \cellcolor{green!25}\textbf{23.0} & 0.0 & \cellcolor{green!25}\textbf{12.2} & 0.0 \\
    Umaze-Diverse & 0.0 & 0.0 & 0.0 & 0.0 \\
    Medium-Play & \cellcolor{green!25}\textbf{13.4} & 0.0 & 0.0 & 0.0 \\
    Medium-Diverse & \cellcolor{green!25}\textbf{19.0} & 0.0 & \cellcolor{green!25}\textbf{7.7} & 0.0\\
    Large-Play & 0.0 & 0.0 & 0.0 & 0.0 \\
    Large-Diverse & 0.0 & 0.0 & 0.0 & 0.0 \\
    \midrule
    \textbf{Average Score} & \cellcolor{green!25}\textbf{9.2} & 0.0 & \cellcolor{green!25}\textbf{3.3} & 0.0 \\
    \bottomrule
    \end{tabular}
    \label{tab:antmaze}
\end{table}

Compared with MuJoCo datasets used and evaluated in the main text, Antmaze datasets from D4RL benchmark are much more challenging for model-based offline RL algorithms. To examine whether SUMO can also benefit base algorithms in Antmaze tasks, we conduct extensive experiments on top of MOPO and AMOReL on 6 Antmaze datasets: \textit{antmaze-umaze}, \textit{antmaze-umaze-diverse}, \textit{antmaze-medium-diverse}, \textit{antmaze-medium-play}, \textit{antmaze-large-diverse}, \textit{antmaze-large-play}, and compare the final performance between base algorithms and with SUMO added. The version of Antmaze datasets we use is v0. 

The experimental results are presented in Table~\ref{tab:antmaze}. It is clear that MOPO and AMOReL both struggle in all Antmaze tasks, achieving a score of zero. After integrating with SUMO, the performance of MOPO and AMOReL on some datasets has raised, indicating the agent learns a meaningful policy. This shows the efficacy of SUMO in challenging Antmaze domains.

\subsection{D.2. More Comparison with Other Uncertainty Estimation Methods}

In the main text, we show the superiority of SUMO over common model ensemble-based methods. In this part, we further compare SUMO with other uncertainty estimation methods for model-based offline RL from more recent literatures. We list our baselines as follows:

\noindent \textbf{Count-based Estimation (CE).} The method used by~\cite{kim2023model}, which utilizes the count estimates of state-action pairs to quantify the model uncertainty.

\noindent \textbf{Riemannian Pullback Metric (RPM).} The method used by~\cite{tennenholtz2022uncertainty}. It estimates the sample uncertainty by computing its KNN distance to the data in the dataset using the Riemannian pullback metric.

We follow the experimental setup in Section \ref{sec:experiment} of the main text, and compare SUMO with CE and RPM on D4RL datasets, using $\rho$ and $r$ as evaluating metrics. The comparison results are listed in Table~\ref{tab:8}. The results reveal that SUMO is comparable to CE and RPM in detecting OOD samples. However, using CE or RPM will bring heavy computational burden and cumbersome algorithmic procedures.

\begin{table}
\caption{Spearman rank ($\rho$) and Pearson bivariate ($r$) correlations comparison of SUMO with CE and RPM. We run each experiment by 5 different seeds and report the average $\rho$ and $r$.}
\label{tab:8}
    \centering
\begin{tabular}{l | cc | cc | cc}
\toprule
\multicolumn{1}{c}{\multirow{2}{*}{\textbf{Task Name}}} &
  \multicolumn{2}{c}{CE} &
  \multicolumn{2}{c}{RPM} & 
  \multicolumn{2}{c}{SUMO} 
  \\ \cline{2-7} 
\multicolumn{1}{c}{} &
  $\mathbf{\rho}$ &
  $r$ &
  $\mathbf{\rho}$ &
  $r$ &
  $\mathbf{\rho}$ &
  $r$ \\ \hline
  half-m & \cellcolor{green!25}\textbf{0.87} & \cellcolor{yellow!50}\textbf{0.79} & 0.80 & 0.74 & 0.84 & 0.77 \\
  hopper-m & 0.81 & 0.72 & 0.83 & 0.78 & \cellcolor{green!25}\textbf{0.86} & \cellcolor{yellow!50}\textbf{0.79} \\
  walker2d-m & \cellcolor{green!25}\textbf{0.86} & \cellcolor{yellow!50}\textbf{0.77} & 0.79 & 0.76 & 0.82 & 0.71 \\
  half-m-e & 0.85 & 0.74 & 0.76 & 0.70 & \cellcolor{green!25}\textbf{0.88} & \cellcolor{yellow!50}\textbf{0.81} \\
  hopper-m-e & \cellcolor{green!25}\textbf{0.77} & 0.63 & 0.70 & 0.65 & 0.75 & \cellcolor{yellow!50}\textbf{0.68} \\
  walker2d-m-e & 0.83 & \cellcolor{yellow!50}\textbf{0.81} & 0.77 & 0.72 & \cellcolor{green!25}\textbf{0.84} & 0.80 \\ \hline
  \textbf{Mean} & \cellcolor{green!25}\textbf{0.83} & 0.74 & 0.78 & 0.73 & \cellcolor{green!25}\textbf{0.83} & \cellcolor{yellow!50}\textbf{0.76} \\
  \bottomrule
  
\end{tabular}

\end{table}

\subsection{D.3. More Ablation Study Results}

\begin{table}[]
    \centering
    \caption{Normalized average score comparison of MOPO+SUMO and AMOReL+SUMO w/ and w/o the reward dimension in the search vector. w/ $r$ means including the reward dimension, and w/o $r$ means excluding the reward dimension. Each experiment is run with 5 random seeds.}
    \begin{tabular}{c|cc|cc}
    \toprule
    \multicolumn{1}{c}{\multirow{2}{*}{\textbf{Task Name}}} & \multicolumn{2}{c}{MOPO+SUMO} & \multicolumn{2}{c}{AMOReL+SUMO} \\
    \cline{2-5}
     & w/ $r$ & w/o $r$ & w/ $r$ & w/o $r$ \\
     \midrule
     half-r & 37.1 & 37.2 & 43.9 & 44.2 \\
     hopper-r & 24.9 & 24.5 & 32.2 & 32.4 \\
     walker2d-r & 13.3 & 13.1 & 19.8 & 21.0 \\
     half-m-r & 73.2 & 73.1 & 76.7 & 76.8 \\
     hopper-m-r & 64.2 & 65.4 & 88.1 & 88.7 \\
     walker2d-m-r & 68.9 & 70.3 & 67.8 & 65.3 \\
     half-m & 66.2 & 68.9 & 84.2 & 82.1 \\
     hopper-m & 76.7 & 74.6 & 94.3 & 95.0 \\
     walker2d-m & 56.9 & 57.3 & 66.3 & 67.4 \\
     half-m-e & 84.3 & 84.1 & 99.6 & 99.4 \\
     hopper-m-e & 88.5 & 88.1 & 100.7 & 101.5 \\
     walker2d-m-e & 82.2 & 81.9 & 113.4 & 112.3 \\
     half-e & 87.8 & 87.1 & 110.2 & 112.3 \\
     hopper-e & 101.9 & 101.2 & 106.5 & 105.4 \\
     walker2d-e & 115.2 & 114.4 & 105.9 & 106.3 \\
     \midrule
     \textbf{Average Score} & 69.4 & 69.4 & 80.6 & 80.7\\
     \bottomrule
    \end{tabular}

    \label{tab:includer}
\end{table}

In this part, we supplement the ablation study results over whether including the reward dimension into the search vector affects the performance of SUMO. We choose $(s\oplus a\oplus s^\prime)$ as the base search vector and examine whether including the reward into the search vector, i.e., $(s\oplus a\oplus r\oplus s^\prime)$ can incur a distinct performance for SUMO.

The base algorithms we choose are MOPO+SUMO and AMOReL+SUMO, we use $(s\oplus a\oplus s^\prime)$ and $(s\oplus a\oplus r\oplus s^\prime)$ as the search vector and conduct experiments on D4RL MuJoCo datasets. The results are shown in Table~\ref{tab:includer}. We can see almost no difference on algorithm performance between these two search vectors, indicating that whether including the reward dimension for KNN search has alomost no effect on the efficacy of SUMO. Therefore, we can simply drop the reward dimension for KNN search.

\subsection{D.4. More Parameter Study Results}
In the main text, we conduct the parameter study for $k$ value of SUMO. In this part, we provide more parameter study results for AMOReL+SUMO and MOPO+SUMO.

For AMOReL+SUMO, we focus on the threshold coefficient $\alpha$. For MOPO+SUMO, we focus on the penalty coefficient $\lambda$ and the sampling coefficient $\eta$. We conduct experiments on MuJoCo datasets from D4RL.

\noindent For AMOReL+SUMO:

\begin{itemize}

    \item \textbf{Sampling coefficient $\eta$:} $\eta$ determines the proportion of dataset samples used for training. A larger $\eta$ means more real samples for training, while a smaller $\eta$ means more synthetic samples for training. We vary the value of $\eta$ to $\{0.05,0.5,0.9\}$, and conduct experiments on \textit{halfcheetah-medium-expert-v2}. The results are shown in Figure~\ref{fig:4} Left. We find that a better performance is achieved when using more real samples for training, for example, $\eta=0.9$. Therefore, in this work, we set $\eta$ to 0.9.
    \item \textbf{Threshold coefficient $\alpha$}: $\alpha$ controls the size of the trajectory truncation threshold. A larger $\alpha$ results in a larger threshold, allowing more synthetic samples to be used for training. A smaller $\alpha$ leads to a smaller threshold, allowing fewer synthetic samples for training. We use \textit{halfcheetah-medium-expert} for experiments, with different $\alpha\in\{1,5,10\}$. The results are shown in Figure~\ref{fig:4} Right. We find that both too large and too small values for $\alpha$ do not perform well. An appropriate value for $\alpha$, such as $\alpha=5$, leads to better performance.

\end{itemize}

\begin{figure}[htb]

    \centering
    \includegraphics[width=1\linewidth]{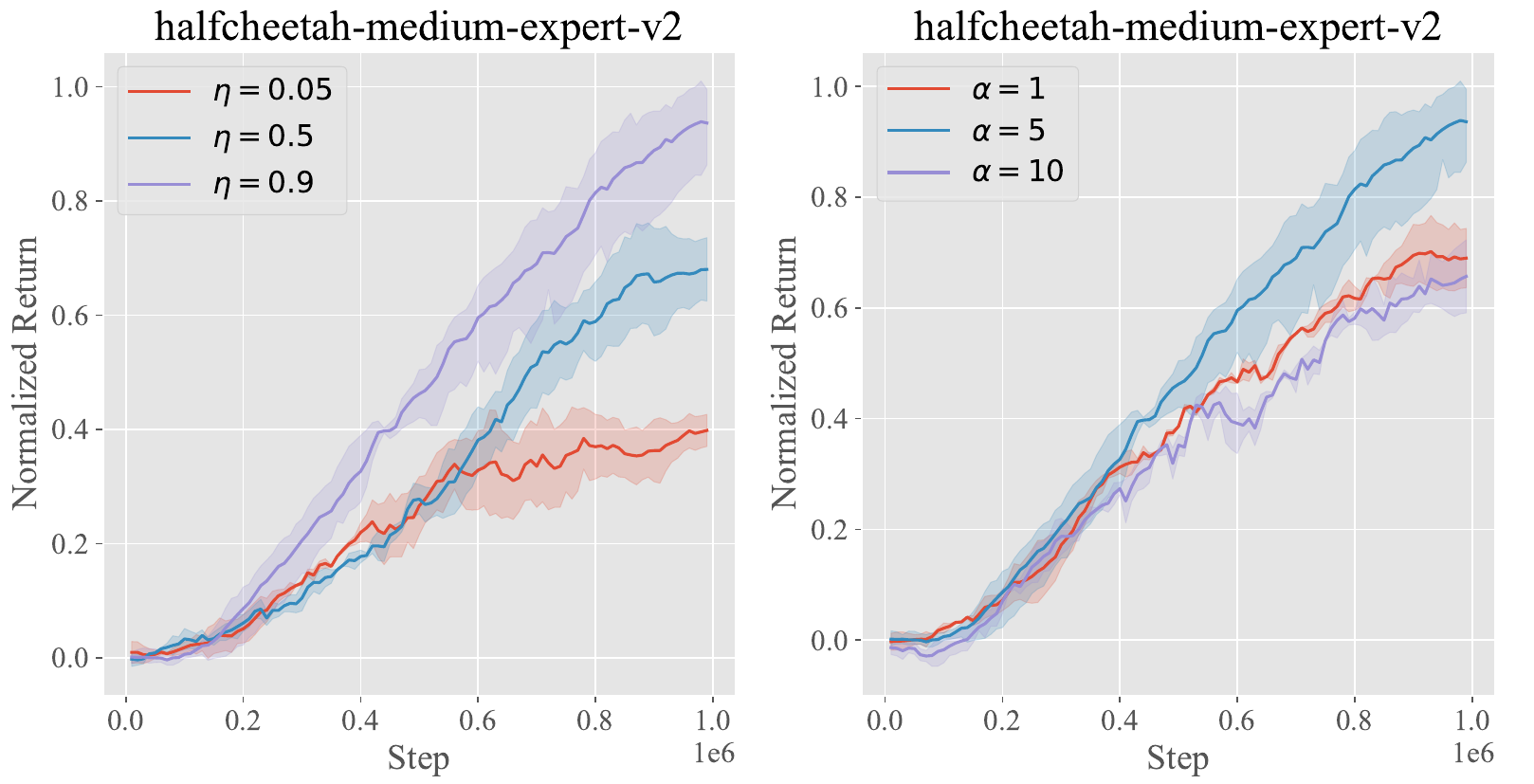}
    \caption{Parameter study on $\alpha$ of AMOReL+SUMO. We run each experiment with 5 random seeds. The shaded region denotes the standard deviation.} 
    \vspace{-0.2cm}
    \label{fig:4}
\end{figure}

\begin{figure}[htb]

    \centering
    \includegraphics[width=1\linewidth]{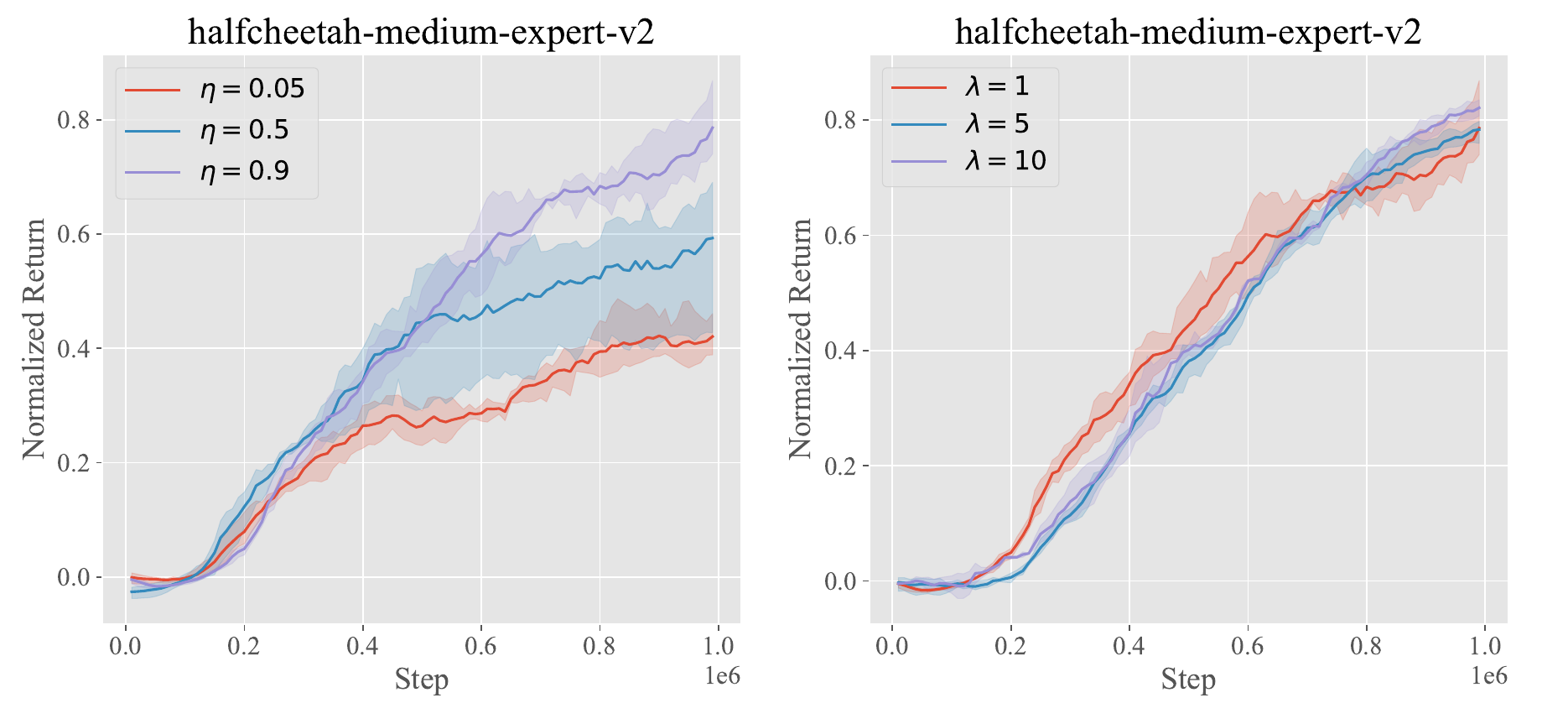}
    \caption{\textbf{Left:} Parameter study on $\eta$ of MOPO+SUMO. \textbf{Right:} Parameter study on $\lambda$ of MOPO+SUMO. We run each experiment with 5 random seeds and the shaded region captures the standard deviation.}
    \vspace{-0.2cm}
    \label{fig:5}
\end{figure}

\noindent For MOPO+SUMO:
\begin{itemize}
    \item \textbf{Sampling coefficient $\eta$:} Similar to AMOREL+SUMO, $\eta$ controls the proportion of real samples for training. We conduct experiments on \textit{halfcheetah-medium-expert-v2}, with different $\eta\in\{0.05,0.5,0.9\}$, and present the experimental results in Figure~\ref{fig:5} Left. We find that using larger $\eta$ leads to a better performance.
    \item \textbf{Penalty coefficient $\lambda$:} $\lambda$ determines the degree of reward penalty. A larger $\lambda$ means a greater penalty for synthetic samples. We vary $\lambda$ to $\{1,5,10\}$, and conduct experiments on \textit{halfcheetah-medium-expert-v2}. The results are shown in Figure~\ref{fig:5} Right. We find that MOPO-SUMO is robust to the value of $\lambda$. Therefore, in this work, we set $\lambda=1$.
\end{itemize}

\subsection{D.5. Full Comparison with More Baselines}
In the main text, we have compared the performance of base algorithms with and without applying SUMO for uncertainty estimation. These base algorithms (MOPO, AMOReL, MOReL and MOBILE) all need to estimate and utilize uncertainty. In this part, we add more baselines which discard uncertainty estimation for comparison. We choose COMBO and RAMBO as baselines and conduct experiments on 15 D4RL MuJoCo datasets. The full experimental results are listed in Table~\ref{tab:fullcomparison}. We can see SUMO can bring performance bonus to base algorithms, surpassing the compared baselines. Especially, the average score of AMOReL is lower than that of COMBO and RAMBO, but AMOReL+SUMO shows a superior performance to COMBO and RAMBO.

\begin{sidewaystable*}
\caption{Normalized average score comparison of different base algorithms with and without SUMO, as well as other model-based offline algorithms including COMBO and RAMBO on 15 D4RL MuJoCo datasets, and the version of datasets we use is "-v2". We abbreviate "halfcheetah" as "half", "random" as "r", "medium" as "m", "medium-replay" as "m-r", "medium-expert" as "m-e" and "expert" as "e". We run each algorithm for 1M gradient steps with 5 random seeds. We report the final average performance and $\pm$ captures the standard deviation. Bold numbers with a green background represent the best average scores within each group.}
\label{tab:fullcomparison}
\centering
\begin{tabular}{c|cc|cc|cc|cc|cccc}
\toprule
\multicolumn{1}{c}{\multirow{2}{*}{\textbf{Task Name}}}  & \multicolumn{2}{c}{MOPO}  & \multicolumn{2}{c}{AMOReL} & \multicolumn{2}{c}{MOReL} & \multicolumn{2}{c}{MOBILE} & \multicolumn{1}{c}{\multirow{2}{*}{COMBO}} & \multicolumn{1}{c}{\multirow{2}{*}{RAMBO}}\\ 
\cline{2-9}

 & SUMO & Base & SUMO & Base & SUMO & Base & SUMO & Base & & \\
\midrule

half-r  & \cellcolor{green!25} \textbf{37.2}$\pm$1.9  & 34.9$\pm$1.4 & \cellcolor{green!25} \textbf{44.2}$\pm$2.1 & 31.8$\pm$2.4 & \cellcolor{green!25}\textbf{37.3$\pm$2.1} & 29.8$\pm$1.2 & 34.9$\pm$2.1 & \cellcolor{green!25}\textbf{37.8$\pm$2.9} & 34.3$\pm$1.9 & 39.7$\pm$1.9\\
hopper-r  & \cellcolor{green!25} \textbf{24.5}$\pm$0.9  & 19.4$\pm$0.7 & 29.7$\pm$0.6 & \cellcolor{green!25} \textbf{32.4$\pm1.2$} & \cellcolor{green!25}\textbf{33.2$\pm$0.7} & 30.1$\pm$1.0  & 30.8$\pm$0.9 & \cellcolor{green!25}\textbf{32.6 $\pm$ 1.2} & 25.6$\pm$0.8 & 21.6$\pm$6.9\\
walker2d-r           & 11.4$\pm$1.3  & \cellcolor{green!25} \textbf{13.1$\pm$1.1}  &    20.3$\pm$0.2          & \cellcolor{green!25} \textbf{21.0$\pm$0.3} & 17.8$\pm$0.6 & \cellcolor{green!25}\textbf{19.4$\pm$0.3} & \cellcolor{green!25}\textbf{27.9$\pm$2.0} & 16.3$\pm$ 4.6 & 11.3$\pm$0.8 & 13.4$\pm$5.4\\
half-m-r & \cellcolor{green!25} \textbf{73.1}$\pm$2.1  & 65.0$\pm$3.3 & \cellcolor{green!25} \textbf{76.8}$\pm$2.7   & 49.6$\pm$2.3 & \cellcolor{green!25}\textbf{67.9$\pm$2.5} &51.2$\pm$1.9  & \cellcolor{green!25}\textbf{76.2$\pm$1.3} & 67.9$\pm$2.0 & 51.0$\pm$2.3 & 67.1$\pm$1.2\\
hopper-m-r      & \cellcolor{green!25} \textbf{65.4}$\pm$3.2 & 38.8$\pm$2.4 & \cellcolor{green!25} \textbf{88.7}$\pm$1.3   & 80.5$\pm$1.0 & \cellcolor{green!25}\textbf{83.9$\pm$1.3} &76.3$\pm$1.0 & \cellcolor{green!25}\textbf{109.9$\pm$1.4} & 104.9$\pm$0.9  & 82.2$\pm$2.1 & 92.6$\pm$4.5\\
walker2d-m-r    & 70.3$\pm$0.6  & \cellcolor{green!25} \textbf{74.8$\pm$1.5} & \cellcolor{green!25} \textbf{65.3}$\pm$2.7   & 46.0$\pm$1.9 & \cellcolor{green!25}\textbf{61.3$\pm$3.1} &48.1$\pm$4.2 & 78.2$\pm$1.5 & \cellcolor{green!25}\textbf{83.9$\pm$1.3}  & 83.8$\pm$1.1 & 83.4$\pm$9.7\\
half-m        & 68.9$\pm$2.3  & \cellcolor{green!25} \textbf{73.1$\pm$2.7} & \cellcolor{green!25} \textbf{82.1}$\pm$2.8   & 69.2$\pm$1.2 & 57.9$\pm$1.2 &\cellcolor{green!25}\textbf{62.4$\pm$1.3} & \cellcolor{green!25}\textbf{84.3$\pm$2.4} & 75.1$\pm$1.5 & 59.5$\pm$0.9 & 75.1$\pm$0.8\\
hopper-m  & \cellcolor{green!25} \textbf{74.6}$\pm$1.9 & 45.6$\pm$2.5 & \cellcolor{green!25} \textbf{95.0}$\pm$2.1  & 87.2$\pm$3.4 & 82.1$\pm$1.4 &\cellcolor{green!25}\textbf{84.7$\pm$3.1} & \cellcolor{green!25}\textbf{104.8$\pm$2.1} & 102.9$\pm$1.9  & 83.6$\pm$2.4 & 92.4$\pm$6.7\\
walker2d-m           & \cellcolor{green!25} \textbf{57.3}$\pm$1.6  & 42.3$\pm$0.8 & 67.4$\pm$0.9   & \cellcolor{green!25} \textbf{71.2$\pm$1.3}& \cellcolor{green!25}\textbf{77.1$\pm$3.5} &67.6$\pm$2.2 & \cellcolor{green!25}\textbf{94.1$\pm$2.5} & 89.1$\pm$1.0   & 89.8$\pm$1.8& 87.7$\pm$2.1\\
half-m-e & \cellcolor{green!25} \textbf{84.1}$\pm$1.4  & 76.6$\pm$1.0 & \cellcolor{green!25} \textbf{99.4}$\pm$3.6 & 90.6$\pm$2.1 & \cellcolor{green!25}\textbf{98.6$\pm$3.5} & 92.3$\pm$4.6 & 106.6$\pm$2.4 & \cellcolor{green!25}\textbf{109.2$\pm$3.8}  & 84.0$\pm$1.2 & 92.1$\pm$3.5\\
hopper-m-e  & \cellcolor{green!25} \textbf{88.1}$\pm$1.9 & 69.1$\pm$1.2 & 101.5$\pm$0.4 & \cellcolor{green!25} \textbf{106.2$\pm$1.5} & \cellcolor{green!25}\textbf{105.8$\pm$1.4} &102.4$\pm$0.9 & 107.8$\pm$0.7 & \cellcolor{green!25}\textbf{110.1$\pm$1.3}  & 104.6$\pm$3.4 & 83.1$\pm$7.8\\
walker2d-m-e    &\cellcolor{green!25}  \textbf{81.9}$\pm$1.6 & 75.4$\pm$1.1 & \cellcolor{green!25} \textbf{109.6}$\pm$0.7  & 92.3$\pm$0.9 & 86.1$\pm$1.8 & \cellcolor{green!25}\textbf{90.4$\pm$1.4} & \cellcolor{green!25}\textbf{122.8$\pm$0.4} & 115.9$\pm$0.8  & 98.1$\pm$0.4 & 62.1$\pm$13.4\\
half-e & 87.1$\pm$1.2 & \cellcolor{green!25} \textbf{88.7$\pm$1.6} & \cellcolor{green!25} \textbf{112.3}$\pm$2.5 & 103.2$\pm$1.9 & \cellcolor{green!25}\textbf{109.9$\pm$2.1} &105.8$\pm$1.6 & 111.5$\pm$1.5 & \cellcolor{green!25}\textbf{113.1$\pm$2.1} & 100.4$\pm$0.2 & 107.5 $\pm$ 11.6 \\
hopper-e & \cellcolor{green!25} \textbf{101.2}$\pm$ 1.8 & 83.9$\pm$0.7 & \cellcolor{green!25} \textbf{105.4}$\pm$0.6 & 94.5$\pm$0.3 & \cellcolor{green!25}\textbf{101.8$\pm$1.9} &92.5$\pm$1.0 & \cellcolor{green!25}\textbf{115.9 $\pm$ 2.9} & 112.4$\pm$3.5  & 111.4$\pm$1.5 & 114.2$\pm$17.9\\
walker2d-e & \cellcolor{green!25} \textbf{114.4}$\pm$1.1 & 95.3$\pm$3.4 & 106.3$\pm$1.3 & \cellcolor{green!25} \textbf{107.2$\pm$1.0} & 106.2$\pm$1.5 &\cellcolor{green!25}\textbf{108.3$\pm$2.1} & \cellcolor{green!25}\textbf{116.3$\pm$1.5} & 113.7$\pm$1.1  & 108.2$\pm$2.0 & 102.3$\pm$3.1\\
\midrule
\textbf{Average score} & \cellcolor{green!25} \textbf{69.3} & 59.7 & \cellcolor{green!25} \textbf{80.3} & 72.2 &
\cellcolor{green!25}\textbf{75.1} &70.7 & \cellcolor{green!25}\textbf{88.2} & 85.6 & 73.7 & 75.6\\
\bottomrule
\end{tabular}
\end{sidewaystable*}

\end{document}